\documentclass[letterpaper]{article} 
\usepackage[]{aaai2026} 
\usepackage{times} 
\usepackage{helvet} 
\usepackage{courier} 
\usepackage[hyphens]{url} 
\usepackage{graphicx} 
\usepackage{natbib} 
\usepackage{caption} 
\usepackage{array} 

\usepackage{amsthm}
\usepackage{tikz}
\usepackage{pgfplots}
\pgfplotsset{compat=1.18}
\usetikzlibrary{arrows.meta}

\usepackage{algorithm}
\usepackage[noend]{algpseudocode}
\usepackage{ifthen}
\usepackage{xspace}

\algrenewcommand\algorithmicindent{0.5em} 

\usepackage{multirow}
\usepackage{booktabs}
\usepackage{amsmath}
\usepackage{amssymb}
\usepackage{makecell}
\usepackage{subcaption}
\nocopyright

\newtheorem{definition}{Definition}
\newtheorem{theorem}{Theorem}
\newtheorem{lemma}{Lemma}
\newtheorem{prop}{Property}

\newcommand\algname[1]{\textsf{#1}\xspace}
\newcommand\namoa{\algname{NAMOA*}}
\newcommand\namoadr{\algname{NAMOA*dr}}
\newcommand\namoamdr{\algname{NAMOA*dr-mvh}}
\newcommand\lnamoamdr{\algname{L-NAMOA*dr-mvh}}

\newcommand{\frankenmvh}{\lnamoamdr}
\newcommand{\soundmvh}{\namoamdr}

\newcommand\apex{\algname{A*pex}}

\newcommand\nd{\mathrm{Nd}\xspace}

\newcommand{\open}{\textsc{Open}\xspace}
\newcommand{\mvh}{\textsc{MVH}}

\newcommand{\hip}{\vecv{h}^{\mathrm{IP}}}

\newcommand{\sols}{\mathrm{Sols}}

\newcommand{\vecv}[1]{\mathbf{#1}}

\newcommand{\weakdom}{\preceq}

\newcommand{\Tr}{\mathrm{Tr}}

\newcommand{\sstart}{s_{\mathrm{start}}}
\newcommand{\sgoal}{s_{\mathrm{goal}}}

\newcommand{\GTcl}[1]{G^{\Tr}_{\mathrm{cl}}(#1)}
\newcommand{\GTclbare}{G^{\Tr}_{\mathrm{cl}}}

\newcommand{\lexprec}{\prec_{\mathrm{lex}}}
\newcommand{\lexsucc}{\succ_{\mathrm{lex}}}

\newcommand{\OS}[1]{{\color{teal}{OS: #1}}}
\newcommand{\ignore}[1]{}
\newcommand{\Continue}{\State \textbf{continue}}

\newcommand{\panda}{\texttt{panda-RRG-8}}
\newcommand{\ny}{\texttt{NY-3obj}}
\newcommand{\netm}{\texttt{netM-10}}

\newif\ifarxivappendix
\arxivappendixtrue
\ifdefined\withoutappendixversion
  \arxivappendixfalse
\fi
\ifdefined\withappendixversion
  \arxivappendixtrue
\fi

\ifarxivappendix

  \newcommand{\soundmvhdefinitionnote}{formally defined in~\soundmvhappendixref}
  \newcommand{\soundmvhproofnote}{Proof is provided in~\soundmvhappendixref.}
  \newcommand{\frankenmvhproofnote}{Formal proofs for the correctness of the \textsc{ChooseH} and \textsc{LocalDomCheck} procedures, as well as the full proof of Thm.~\ref{thm:frankenmvh-correctness}, are provided in~\frankenmvhappendixref.}
\else
  \newcommand{\soundmvhdefinitionnote}{Due to lack of space, the formal definition is available in~\citeauthor{wolff2026bridging}~\shortcite{wolff2026bridging}}
  \newcommand{\soundmvhproofnote}{Proof is provided in~\citeauthor{wolff2026bridging}~\shortcite{wolff2026bridging}.}
  \newcommand{\frankenmvhproofnote}{Formal proofs for the correctness of the \textsc{ChooseH} and \textsc{LocalDomCheck} procedures, as well as the full proof of Thm.~\ref{thm:frankenmvh-correctness}, are provided in~\citeauthor{wolff2026bridging}~\shortcite{wolff2026bridging}.}
\fi

\title{Bridging Multi-Valued Heuristics and Dimensionality Reduction in Multi-Objective Search}

\author{
    Maya Wolff,\textsuperscript{\rm 1}
    Ariel Felner,\textsuperscript{\rm 2}
    Oren Salzman\textsuperscript{\rm 1}
}

\affiliations{
    \textsuperscript{\rm 1}Department of Computer Science, Technion -- Israel Institute of Technology, Haifa, Israel\\
    \textsuperscript{\rm 2}Faculty of Computer and Information Science, Ben-Gurion University of the Negev, Beer-Sheva, Israel\\
    wo@cs.technion.ac.il, felner@bgu.ac.il, osalzman@cs.technion.ac.il
}

\begin{document}
\setcounter{secnumdepth}{2}
\maketitle

\begin{abstract}

Multi-objective shortest-path (MOSP) algorithms traditionally rely on single-valued heuristics (SVHs), which associate each state with a single admissible cost vector. While SVHs provide safe lower bounds, they fail to capture the trade-off structure of the Pareto frontier and often yield weak search guidance. Multi-valued heuristics (MVHs) address this limitation by mapping states to sets of cost estimates, enabling a richer approximation of possible trade-offs.

Modern MOSP algorithms are highly dependent on dimensionality reduction (DR) techniques to efficiently perform dominance checks. However, integrating MVHs with DR introduces subtle correctness challenges. We show that naively combining DR with MVHs destroys the ordering invariants required for DR, leading to unsound and incomplete search. To address this issue, we develop the first theoretical frameworks for safely integrating MVHs with DR.

First, we introduce \namoamdr, a theoretical baseline that restores search correctness by enforcing heuristic consistency. Recognizing the practical limitations of this approach, we then introduce our primary contribution \frankenmvh. This algorithm employs a ``lazy,'' optimistic approach to DR, preserving exact correctness with only an \emph{admissible} MVH by dynamically detecting and repairing local ordering violations.
Across a range of benchmarks, L-NAMOA*dr-mvh matches or improves over state-of-the-art MOSP algorithms, and achieves speedups of over 10x in instances where the additional guidance provided by the MVH translates into stronger pruning.

\end{abstract}

\section{Introduction \& Related Work}
In the multi-objective shortest-path (MOSP) problem we are interested in finding paths between two vertices of a graph while considering multiple, often conflicting objectives. 
Applications of MOSP range from transporting hazardous materials considering travel distance and risk~\cite{bronfman2015maximin} and inspecting a region of interest using cameras placed on-board robotic platforms~\cite{FuKSA23}. This family of problems is not new, with results dating back decades~\cite{vincke1976,hansen-1980,climaco2012multicriteria,current1993multiobjective,skriver2000classification,ulungu1991multi}. Nevertheless, there has been renewed interest and significant progress in the field of heuristic search for multi-objective search (MOS)~\cite{SalzmanF0ZCK23}, reflecting the reality that real-world systems rarely optimize a single measure \cite{salzman2025multi}.

State-of-the-art (SOTA) heuristic MOS algorithms rely primarily on two mechanisms to improve their efficiency. First, they utilize \emph{heuristics} to guide the search. Most heuristic MOS algorithms use a single-valued heuristic (SVH), which estimates the cost to the goal from every state (see, e.g.,~\cite{ahmadi2024nwmoa,RenHLFKSRC25}). However, there are a handful of works that considered the more-informative notion of multi-valued heuristics (MVH), which map states to sets of cost estimates, enabling a richer approximation of possible trade-offs (see, e.g.,~\cite{namoa,geisser_admissible,ZhangSFKSUK23}). Second, modern algorithms employ \emph{dimensionality reduction} (DR) techniques to improve the efficiency of dominance checking operations (i.e., assessing if one solution is strictly better than another across all problem objectives). Improving these dominance checks is crucial, as they are often the key computational bottleneck of many SOTA algorithms for MOSP~\cite{pulido_dimred}.

Each mechanism (namely, MVHs and DR) by itself is a powerful algorithmic tool. Yet, utilizing both simultaneously is highly non-trivial and has been identified as a key challenge in the field \cite{SalzmanF0ZCK23}. 
To this end, in this paper we pinpoint the precise source of this incompatibility and develop two algorithmic frameworks to overcome it.

First, we introduce \namoamdr which serves as a theoretical baseline. It demonstrates that if we impose constraints on the MVH (specifically, requiring it to be consistent) a relaxed approach of DR can be used safely. However, while theoretically sound, this approach may introduce computational inefficiencies as the algorithm may need to generate redundant paths. Moreover, the approach relies on the difficult task of constructing consistent MVHs.

To overcome these practical bottlenecks, we introduce our primary contribution: \frankenmvh. Rather than forcing the aforementioned constraints upfront, this algorithm adopts a ``lazy,'' optimistic approach. It performs fast DR by default, but dynamically monitors the search to detect cases when the heuristic guidance might cause an error. Only when a potential error is detected does the algorithm fall back to a rigorous, standard dominance check. 

Consequently, \frankenmvh\ preserves soundness and completeness while requiring only a standard, admissible MVH. By avoiding the potential overhead of \namoamdr, it successfully combines the stronger guidance of MVHs with the computational efficiency of DR. We empirically evaluate our algorithms on a range of benchmarks, demonstrating that \frankenmvh\ matches or improves upon \namoadr, with speedups of over $10\times$ when MVH guidance enables stronger pruning.



\section{Notation \& Algorithmic Background}
\label{sec:background}

\subsection{Notation}
\label{sec:notation}

We follow standard notation in MOS~\cite{SalzmanF0ZCK23}. Boldface font indicates vectors, and lower-case and upper-case symbols indicate elements and sets, respectively. For an $N$-dimensional vector $\vecv v$, $v_i$ denotes its $i$-th component. Vector addition is defined as element-wise summation.

Let $\mathbf{p}$ and $\mathbf{q}$ be $N$-dimensional vectors. We say that $\mathbf{p}$ \emph{weakly dominates} $\mathbf{q}$, denoted as $\mathbf{p} \preceq \mathbf{q}$, if $p_i \le q_i$ for all $i=1,\dots,N$. We say that $\mathbf{p}$ \emph{dominates} $\mathbf{q}$, denoted as $\mathbf{p} \prec \mathbf{q}$, if $\mathbf{p}$ weakly dominates $\mathbf{q}$ and there exists an index $j$ such that $p_j < q_j$. When $\mathbf{p} \nprec \mathbf{q}$ and $\mathbf{q} \nprec \mathbf{p}$, we say that $\mathbf{p}$ and~$\mathbf{q}$ are mutually non-dominated. Furthermore, we say that~$\mathbf{p}$ is lexicographically smaller than $\mathbf{q}$, denoted as $\mathbf{p} <_{\rm lex} \mathbf{q}$, if $p_k < q_k$ for the first index $k$ such that $p_k \ne q_k$.

Let $X$ be a set of $N$-dimensional vectors. We denote by $\nd(X)$ the cost-unique subset of $X$ containing only mutually non-dominated vectors. The truncation function $\Tr$ maps a vector $\mathbf{v} = (v_1, \dots, v_N)$ to the $(N-1)$-dimensional vector obtained by removing its first component, i.e., $\mathrm{Tr}(\mathbf{v}) = (v_2, \dots, v_N)$. With a slight abuse of notation, we define the truncated set of $X$ as $\Tr(X) = \nd(\{\Tr(\mathbf{x}) \mid \mathbf{x} \in X\})$, similarly written as $X^{\Tr}$.

\label{notation:t_discarding}
A vector $\vecv{v}\in\mathbb{R}^N$ is \emph{$t$-discarded} by a set $X\subseteq\mathbb{R}^N$ if there exists a vector $\vecv{u}\in X$ such that $\Tr(\vecv{u})\in \Tr(X)$ and one of the following holds:
(i)~$u_1 < v_1$ and $\Tr(\vecv{u}) \preceq \Tr(\vecv{v})$, or
(ii)~$u_1 = v_1$ and $\Tr(\vecv{u}) \prec \Tr(\vecv{v})$.


A \emph{multi-objective search graph} is a tuple $G = \langle S, E, c \rangle$, where $S$ is a finite set of states, $E \subseteq S \times S$ is a finite set of edges, and $c : E \rightarrow \mathbb{R}_{\ge 0}^N$ is a cost function that associates a vector of $N$ non-negative cost components (namely, the objectives) with each edge. The successor function is defined as $\mathrm{Succ}(s) = \{ s' \in S \mid (s,s') \in E \}$. A path $\pi$ from state~$s_1$ to state $s_\ell$ is a sequence of states $[s_1,\dots,s_\ell]$ such that $(s_j,s_{j+1}) \in E$ for all $j=1,\dots,\ell-1$. The cost of a path $\pi$ is $\mathbf{c}(\pi) = \sum_{j=1}^{\ell-1} c(s_j,s_{j+1})$.

A \emph{multi-objective search instance} is a tuple $P = \langle S,E,c,\sstart,\sgoal \rangle$, where $\sstart \in S$ and $\sgoal \in S$ are the start and goal states, respectively. A path connecting $\sstart$ to $\sgoal$ is called a \emph{solution}. A solution is \emph{Pareto-optimal} if its cost is not dominated by any other solution. The \emph{Pareto-optimal solution set}, denoted $\Pi^*$, is the set of all Pareto-optimal solutions.

Finally, multi-objective search algorithms often use a heuristic function to guide the search.

A \emph{single-valued heuristic} (SVH) is a function $h:S\rightarrow\mathbb{R}_{\ge0}^N$ that estimates the cost to $\sgoal$ from every state $s$. 
We say that $h$ is admissible if $h(s)\preceq \vecv c(\pi)$ for every state $s$ and for all path $\pi$ from $s$ to $\sgoal$.
We say that $h$ is consistent if $h(\sgoal) = \mathbf{0}$ and $h(s) \preceq c(s,s') + h(s')$ for all $(s,s') \in E$.

A \emph{multi-valued heuristic} (MVH) $H:S\rightarrow 2^{\mathbb{R}_{\ge0}^N}$ maps each state to a set of mutually non-dominated vectors.
We say that $H$ is admissible if for every state $s$ and for all path $\pi$ from $s$ to $s_\text{goal}$, $H(s)$ contains a cost vector that weakly dominates $\vecv{c}(\pi)$.
We say that $H$ is consistent if $H(\sgoal) = \{\mathbf{0}\}$ and for every $(s, s')\in E$ and for every $\mathbf{h} \in H(s)$, there exists $\mathbf{h}' \in H(s')$ such that $\mathbf{h} \preceq c(s,s') + \mathbf{h}'$.


\subsection{The Pruning Power of MVHs}
\label{sec:mvh-pruning}

Intuitively, the advantage of an MVH over an SVH lies in its ``informativeness''. While an SVH provides one lower-bound estimate, an MVH can capture a set of non-dominated trade-offs that more accurately reflects the costs to the goal.

This is illustrated in Fig.~\ref{fig:mvh_vs_svh_origin} where an SVH (blue) provides a ``loose'' lower bound. While this point dominates a large area of the cost space, its proximity to the origin makes it a weak filter by the actual costs of known solutions. 
In contrast, the MVH provides a frontier of ``tighter'' estimates (red) that sit further from the origin. Since MVH estimates are tighter while remaining admissible, they make it more likely that a path cost will be dominated by an existing solution. This allows the search to prune suboptimal branches early, whereas a looser bound would be forced to explore them.

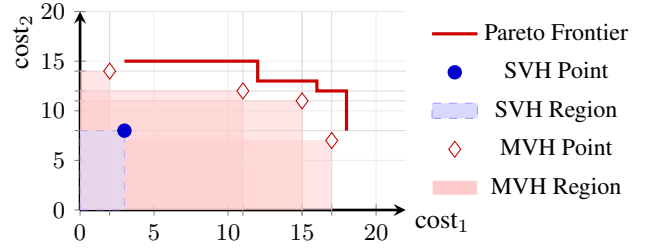
\begin{figure}[t]
\centering
\begin{tikzpicture}
\begin{axis}[
    width=0.7\linewidth, 
    height=0.5\linewidth,
    xmin=0, xmax=22,
    ymin=0, ymax=20,
    axis lines=left,
    xlabel={$\text{cost}_1$},
    ylabel={$\text{cost}_2$},
    ylabel style={at={(axis cs:-1.7,22)}, anchor=south east, xshift=-5pt, yshift=5pt},
    xlabel style={at={(axis cs:22,2)}, anchor=north west, yshift=-5pt},
    tick label style={font=\footnotesize}, 
    xtick={0,5,10,15,20},
    ytick={0,5,10,15,20},
    extra x ticks={2,11,15,17},
    extra y ticks={8,11,12,14},
    extra x tick labels={}, 
    extra y tick labels={},
    extra tick style={grid=major, grid style={line width=.2pt, draw=gray!30}},
    grid=major,
    grid style={line width=.1pt, draw=gray!15},
    axis line style={line width=1pt, -stealth}, 
    enlargelimits=false,
    clip=false,
    legend style={
        draw=white,
        font=\small,
        at={(1.05, 0.5)}, 
        anchor=west,       
        row sep=0.1cm
    }
]

\fill[red!20, opacity=0.5] (axis cs:0,0) rectangle (axis cs:2,14);
\fill[red!20, opacity=0.5] (axis cs:0,0) rectangle (axis cs:11,12);
\fill[red!20, opacity=0.5] (axis cs:0,0) rectangle (axis cs:15,11);
\fill[red!20, opacity=0.5] (axis cs:0,0) rectangle (axis cs:17,7);

\fill[blue!15, draw=blue!40, dashed, thin, opacity=0.6] (axis cs:0,0) rectangle (axis cs:3,8);


\addlegendimage{draw=red!80!black, line width=1.2pt}
\addlegendentry{Pareto Frontier}

\addplot[
    const plot,
    draw=red!80!black,
    line width=1.2pt,
    no marks, 
    forget plot
] coordinates {
    (3, 15) 
    (12, 13)
    (16, 12)
    (18, 8) 
};

\addplot[
    only marks,
    mark=*,
    mark size=2.5pt,
    draw=blue!80!black,
    fill=blue!80!black
] coordinates {(3,8)};
\addlegendentry{SVH Point}
\addlegendimage{area legend, fill=blue!15, draw=blue!40, dashed}
\addlegendentry{SVH Region}

\addplot[
    only marks, 
    mark=diamond*,
    mark size=3pt,
    mark options={fill=white, draw=red!80!black, solid},
    forget plot
] coordinates {
    (2,14)
    (11,12)
    (15,11)
    (17,7)
};

\addlegendimage{only marks, mark=diamond*, mark size=3pt, mark options={fill=white, draw=red!80!black, solid}}
\addlegendentry{MVH Point}
\addlegendimage{area legend, fill=red!20}
\addlegendentry{MVH Region}

\end{axis}
\end{tikzpicture}
\caption{Comparison of pruning regions by SVH (blue) and MVH (red). The solid red line represents the Pareto-optimal solution frontier of the problem.}
\label{fig:mvh_vs_svh_origin}
\end{figure}

\subsection{Algorithmic background}

\subsubsection{Multi-Objective A* (MOSA*).} Best-first search MOS algorithms, often generalized as MOSA*~\cite{SkylerSAFSC0K0U24}, compute $\Pi^*$ by maintaining a priority queue \open{} of generated but not yet extracted paths. A fundamental difference between standard single-objective A* and MOSA* is path tracking: while standard A* typically expands a state only once via its shortest path, MOSA* must maintain a separate search node for every distinct, mutually non-dominated cost vector $\vecv{g}$ reaching that state. When utilizing a traditional single-valued heuristic (SVH), each such node $n$ is evaluated by a single vector $\vecv{f}(n) = \vecv{g}(n) + \vecv{h}(s(n))$. Prominent MOSP algorithms, including \namoa\ and \namoadr, are specific instantiations of this MOSA* framework. However, as we will demonstrate, this straightforward evaluation where one $\vecv{g}$-value maps to exactly one $\vecv{f}$-value becomes more complex when integrating MVHs.

To ensure efficiency, MOSA* algorithms rely on \emph{dominance checks} to determine whether a newly generated or extracted node can still contribute to the Pareto-optimal solution set. This typically involves two layers of pruning: (i) \emph{local dominance checking}, where a node's cost $\vecv{g}$ is compared against previously discovered paths to the same state to prune locally sub-optimal routes, and (ii) \emph{global dominance checking}, where a node's evaluation $\vecv{f}$ is compared against known solutions at the goal to prune paths that are already pruned. Because a node must be compared against entire sets of mutually non-dominated vectors, these dominance checks constitute the primary computational bottleneck of MOS, directly motivating the need for DR.

\ignore{\subsubsection{Best-First Search}
Best-first search MOS algorithms~\cite{SkylerSAFSC0K0U24} compute~$\Pi^*$ while maintaining a priority queue $\open$ of generated but not yet extracted nodes. In MOS (and unlike single-objective A*), multiple search nodes may share the same state $s$ with different cost vectors, as several Pareto-optimal paths may reach~$s$. Each node $n$ stores a state $s(n)$ and an accumulated cost vector~$\vecv{g}(n)$, with evaluation vector $\vecv{f}(n) = \vecv{g}(n) + \vecv{h}(s(n))$.
Initially, $\open$ contains a single node for the start state~$\sstart$ with cost $\vecv{0}$.

At each iteration, the algorithm extracts from $\open$ a node with minimal $\vecv{f}$ according to its selection rule, and performs a dominance check to determine whether the node can still contribute to the solution set. If not, it is discarded. Otherwise, if the node contains a goal state, it is added to the solution set. If not, it is extracted by generating successor nodes, each of which is again subjected to a dominance check before either being discarded or inserted into $\open$. The search terminates when $\open$ becomes empty.

Algorithms such as \namoa and \namoadr\ conform to this framework, differing mainly in the information stored in nodes, the node-selection rule, and the dominance tests they apply.
}

\subsubsection{\namoa} 
\namoa~\cite{namoa} is an instantiation of the MOSA* framework that natively supports MVHs. Instead of mapping to a single $\vecv{f}$-value, a path reaching state $s$ with cost $\vecv{g}$ under an MVH~$H$ is evaluated using a \emph{set} of mutually non-dominated $\vecv{f}$-values, defined as $F(s,\vecv{g}) = \nd(\{\vecv{g}+\vecv{h} \mid \vecv{h} \in H(s)\})$. Consequently, each node in \open{} is represented as a triplet $n=\langle s, \vecv{g}, F(s,\vecv{g})\rangle$.

To execute dominance checks, \namoa\ maintains the following sets. For \emph{global dominance checking}, it maintains $\sols$, the non-dominated cost vectors of solutions found so far. For \emph{local dominance checking}, it maintains two sets of $\vecv{g}$-values for each state $s$: $G_{\mathrm{op}}(s)$ (paths generated but not yet extracted) and $G_{\mathrm{cl}}(s)$ (paths already extracted).\footnote{Here `op` and `cl` correspond to open and closed, respectively.} 

At each iteration, \namoa\ extracts a node whose $F$-set contains an $\vecv{f}$-value that is not dominated by any other $\vecv{f}$-value in \open{}. When a successor node is generated, it is pruned globally if all its $\vecv{f}$-values are dominated by $\sols$, and pruned locally if its $\vecv{g}$-value is dominated by $G_{\mathrm{op}}(s')$ or~$G_{\mathrm{cl}}(s')$. While \namoa\ is optimal in the number of path extractions when using a consistent heuristic~\cite{mandow_consistent}, performing full dominance checks is computationally expensive.

\subsubsection{\namoadr}  
\label{sec:namoadr} 

\begin{algorithm}[t]
\caption{\textcolor{blue}{\namoadr} (\namoamdr{})}
\label{alg:soundmvh}
\begin{algorithmic}[1]
\Statex \textbf{Input:} A MOS instance $(S,E,\vecv c,\sstart,\sgoal)$;
\Statex \hspace{10mm} a consistent \textcolor{blue}{SVH $\vecv h$}. \Comment{\textcolor{red}{MVH $H$}}
\Statex \textbf{Output:} A cost-unique Pareto-optimal solution set $\sols$.

\vspace{2mm}

\State $\text{Sols}\gets\emptyset$
\ForAll{$s\in S$} \Comment{\textcolor{red}{\textbf{for all} $\vecv h\in H(s)$ \textbf{do}}} \label{line:alg1-source-init}
  \State $\textcolor{blue}{\GTcl{s}}\gets \emptyset$ \Comment{\textcolor{red}{$G^{\Tr}_{\mathrm{cl}}(s,\vecv h) \gets \infty$}}
\EndFor

\State $n\gets$ new node with $s(n)=\sstart$
\State $\text{parent}(n)\gets \text{NULL}$;
       $\vecv g(n)\gets \vecv 0$;
       $\vecv f(n)\gets \vecv h(\sstart)$
\State $\open{}\gets\{n\}$

\While{$\open{}\neq\emptyset$}
  \State $n\gets \open{}.\text{Pop}$ \label{line:alg1-lex-smallest-f}
  \Comment{lex smallest $\vecv f$-value}

  \If{$\textcolor{blue}{\GTcl{s(n)}} \weakdom \Tr(\vecv g(n))$} \label{line:alg1-local-dom-check-1}
    \Comment{\textcolor{red}{$G^{\Tr}_{\mathrm{cl}}(s(n), \vecv h(n))$}}
    \Continue \label{alg1line:local-pruning-1}
  \EndIf

  \If{$\textcolor{blue}{\GTcl{\sgoal}} \weakdom \Tr(\vecv f(n))$} \label{line:alg1-global-dom-check-1}
    \Continue \label{line:alg1-global-pruning-1}
  \EndIf

  \State $\textcolor{blue}{\GTcl{s(n)}} \gets
    \nd\!\left(
        \textcolor{blue}{\GTcl{s(n)}} \cup \Tr(\vecv g(n))
    \right)$ \label{line:ctcl-insertion}
  \Statex \Comment{\textcolor{red}{$G^{\Tr}_{\mathrm{cl}}(s(n), \vecv h(n))$}}

  \If{$s(n)=\sgoal$}
    \State $\text{Sols}\gets \text{Sols}\cup\{n\}$
    \Continue
  \EndIf

  \ForAll{$s'\in \mathrm{Succ}(s(n))$}
    \State $n'\gets$ a new node with $s(n')=s'$ \label{line:new-successor}
    \State $\mathrm{parent}(n')\gets n$;
           $\vecv g(n')\gets \vecv g(n)+\vecv c(s(n),s')$

    \State $\vecv f(n')\gets \vecv g(n')+ \textcolor{blue}{\vecv h(s')}$ \label{line:new-successor-end}
    \Statex \Comment{\textcolor{red}{$\forall \vecv h'\in H(s')$ s.t. $\vecv h(s(n))\preceq \vecv c(s(n),s') + \vecv h'$}} \label{line:alg1-successor-generation}

    \If{$\textcolor{blue}{\GTcl{s'}} \weakdom \Tr(\vecv g(n'))$} \label{line:alg1-local-dom-check-2}
      \Comment{\textcolor{red}{$G^{\Tr}_{\mathrm{cl}}(s', \vecv h')$}}
      \Continue \label{line:alg1-local-pruning-2}
    \EndIf

    \If{$\textcolor{blue}{\GTcl{\sgoal}} \weakdom \Tr(\vecv f(n'))$} \label{line:alg1-global-dom-check-2}
      \Continue \label{line:alg1-global-pruning-2}
    \EndIf

    \State $\open{}.\text{Insert}(n')$
      \Comment{ordered lex. according to $\vecv f$-value}
  \EndFor
\EndWhile

\State \Return $\text{Sols}$
\end{algorithmic}
\end{algorithm}

To alleviate the computational bottleneck of full-vector dominance checks, \citeauthor{pulido_dimred}~\shortcite{pulido_dimred} proposed \namoadr. This algorithm restricts the search to use a consistent SVH, enabling the use of dimensionality reduction (DR). In the context of MOS algorithms, DR hinges on the following property:
\begin{prop}    
\label{prop:dr}
    Assume that a best-first search MOS algorithm
    (i) uses a consistent SVH $h$ and
    (ii) orders \open{} lexicographically according to the $\vecv{f}$-value of nodes (i.e., $\vecv{g}+\vecv{h}$).
    Let $n,n'$ be two search nodes corresponding to the same state $s$.
    If $n$ is extracted before $n'$ from \open then both $f_1(n)\leq f_1(n')$ and $g_1(n) \leq g_1(n')$.    
\end{prop}

To understand why Property~\ref{prop:dr} holds, note that nodes are extracted from \open{} in lexicographically non-decreasing order of $\vecv{f}$-values. Because the $\vecv{h}$-value is fixed for a given state $s$ when using an SVH, the nodes of state $s$ are extracted in lexicographically non-decreasing order of $\vecv{g}$-values. 

The implication of Property~\ref{prop:dr} is that the first cost component ($g_1$ or $f_1$) of a newly extracted node will always be greater than or equal to that of all previously extracted nodes for the same state. This guarantees monotonic growth in the first dimension, eliminating the need to check dominance for that cost element. 

Leveraging this, \namoadr\ replaces standard dominance checks with more efficient \emph{$t$-discarding} tests (Sec.~\ref{sec:notation}). Instead of maintaining full open and closed sets, \namoadr\ maintains only a \emph{truncated} set of closed nodes for each state, denoted by $\GTcl{s}$. 

Alg.~\ref{alg:soundmvh} illustrates this control flow (red annotations should be ignored at this point). Blue annotations highlight the specific operations where DR replaces standard \namoa\ dominance tests, utilizing the truncated vectors $\Tr(\cdot)$ and the reduced per-state sets $\GTcl{\cdot}$. 

After initialization, nodes are extracted from $\open{}$ in lexicographically non-decreasing $\vecv{f}$ order (Line~\ref{line:alg1-lex-smallest-f}). Due to Property~\ref{prop:dr}, dominance checks at both extraction time and generation time are performed using the truncated vectors $\Tr(\cdot)$ and the reduced per-state sets $\GTcl{\cdot}$. Specifically, an extracted node $n$ is pruned if its truncated cost vector $\Tr(\vecv{g}(n))$ is dominated by $\GTcl{s(n)}$ (Line~\ref{line:alg1-local-dom-check-1}) or if its truncated evaluation vector $\Tr(\vecv{f}(n))$ is dominated by $\GTcl{\sgoal}$ (Line~\ref{line:alg1-global-dom-check-1}). 

When a node passes these $t$-discarding checks upon extraction, its truncated $\vecv{g}$-value is inserted into $\GTclbare$ (Line~\ref{line:ctcl-insertion}), ensuring that subsequent dominance tests can be done using $N{-}1$ dimensions. This same $t$-discarding logic is applied to filter new generated successor nodes (Lines~\ref{line:alg1-local-dom-check-2} and~\ref{line:alg1-global-dom-check-2}).

\section{The Incompatibility of MVH and DR}
\label{sec:incompatibility}
In the previous section, we established that DR via $t$-discarding relies entirely on Property~\ref{prop:dr}: nodes corresponding to the same state must be extracted in lexicographically non-decreasing order of their $\vecv{g}$-values. In the context of \namoadr, this ordering invariant is naturally guaranteed when using a consistent SVH. 

However, this invariant collapses when introducing MVHs. When a state $s$ is associated with a set of heuristic estimates $H(s)$, two nodes~$n$ and $n'$ of the same state can be evaluated using different heuristic vectors, $\vecv{h} \in H(s)$ and $\vecv{h}' \in H(s)$. 
As we will see shortly, in certain settings, $\vecv{f}(n) \lexprec \vecv{f}(n')$ even if $\vecv{g}(n) \lexsucc \vecv{g}(n')$. When extracting nodes by their $\vecv{f}$-values, a node with a larger first-objective cost $g_1$ might be extracted before a node with a smaller $g_1$ cost, violating monotonic growth in the first dimension.

\begin{lemma}
\label{lem:mvh-breaks-prop1}
Property~\ref{prop:dr} does not hold generally for best-first MOS algorithms employing an MVH $H$, even if $H$ is admissible and consistent.
\end{lemma}

\begin{proof}
We prove this by counter-example. Consider the bi-objective search graph shown in Fig.~\ref{fig:example_graph}. The search query is~$\sstart \to \sgoal$, and the Pareto-optimal solution set $\Pi^*$ consists of four solution paths (colored in Fig.~\ref{fig:example_graph}, also shown in Fig.~\ref{fig:mvh_vs_svh_origin}): $\pi_1, \pi_2, \pi_3$ and $\pi_4$ with costs $(3,15), (12,13), (16,12)$ and $(18,8)$, respectively.

\begin{figure}[t]
  \centering
  \begin{minipage}{0.49\columnwidth}
    \centering
    \resizebox{\linewidth}{!}{
      \begin{tikzpicture}[
        >=Latex,
        every node/.style={circle, draw=black, thick, minimum size=7mm, inner sep=0pt},
        cost/.style={draw=none, fill=white, inner sep=1pt, font=\small},
        hpath/.style={line width=1.2pt, >={Latex[length=1.8mm,width=1.4mm]}}
      ]
          \node (s) at (0,2.5) {$\sstart$};
          \node (g) at (3.6,2.5) {$\sgoal$};
          \node (s1) at (0,0) {$s_1$};
          \node (s2) at (3.6,0) {$s_2$};
          \node (s3) at (1.8,0) {$s_3$};

        \draw[thick] (s) to node[cost, above] {$(3,15)$} (g);
        \draw[thick] (s) to node[cost, left] {$(7,9)$} (s1);
        \draw[thick] (s1) to node[cost, below] {$(1,1)$} (s3);
        \draw[thick] (s3) to node[cost, below] {$(7,1)$} (s2);
        \draw[thick] (s2) to node[cost, right] {$(1,1)$} (g);
        \draw[thick] (s) to node[cost, sloped, above, pos=0.4] {$(10,6)$} (s3);
        \draw[thick] (s3) to node[cost, sloped, above, pos=0.6] {$(2,7)$} (g);

        \draw[hpath, ->, blue] (s.east) to[out=10,in=170] (g.west);
        \draw[hpath, ->, orange!90!black, dashed] (s.south east) to[out=-28,in=118] (s3.north west);
        \draw[hpath, ->, orange!90!black, dashed] (s3.north east) to[out=62,in=-152] (g.south west);
        \draw[hpath, ->, green!50!black, dotted] (s.south west) to[out=-118,in=62] (s1.north east);
        \draw[hpath, ->, green!50!black, dotted] (s1.east) to[out=0,in=180] (s3.west);
        \draw[hpath, ->, green!50!black, dotted] (s3.east) to[out=-6,in=186] (s2.west);
        \draw[hpath, ->, green!50!black, dotted] (s2.north east) to[out=62,in=-118] (g.south west);

      \draw[hpath, ->, red, dash dot]
        (s.south east) to[out=-42,in=132] (s3.north west);
      \draw[hpath, ->, red, dash dot]
        (s3.east) to[out=-18,in=198] (s2.west);
      \draw[hpath, ->, red, dash dot]
        (s2.north east) to[out=48,in=-132] (g.south west);
      \end{tikzpicture}
    }
    \captionof{figure}{Search graph.}
    \label{fig:example_graph}
  \end{minipage}
  \hfill
  \begin{minipage}{0.49\columnwidth}
    \centering
    \small
    \begin{tabular}{ll}
      \toprule
      State & $H(s)$ \\
      \midrule
      $\sstart$ & \makecell{\{(2,14), (11,12), \\~~~(15,11), (17,7)\}} \\
      $s_1$ & \{(2,7), (8,2)\} \\
      $s_2$ & \{(1,1)\} \\
      $s_3$ & \{(1,6), (7,1)\} \\
      $\sgoal$ & \{(0,0)\} \\
      \bottomrule
    \end{tabular}
    \vspace{12pt}
    \captionof{table}{MVH $H$ values.}
    \label{tab:simple_example}
  \end{minipage}
\end{figure}

Table~\ref{tab:simple_example} specifies an admissible and consistent MVH $H$ for this query (also shown in Fig.~\ref{fig:mvh_vs_svh_origin}). We execute a naive adaptation of \namoadr\ using this MVH. Node entries in \open{} are represented as $n=(s, \vecv{g}, F(s,\vecv{g}))$, where $F(s,\vecv{g}) = \nd\{\vecv{g}+\vecv{h}\mid \vecv{h}\in H(s)\}$. 
The node selected for extraction has the lexicographically-smallest $\vecv{f}$-value among all available $\vecv{f} \in F$ in \open{}. For DR, we maintain a truncated closed-set $\GTcl{s}$ for each state. As this is a bi-objective graph, truncated vectors are scalars (as they contain only the second cost component).

\vspace{0.5\baselineskip}
\begin{table}[t]
\centering
\footnotesize
\setlength{\tabcolsep}{2.5pt}
\renewcommand{\arraystretch}{0.97}
\caption{Execution trace of naively combining \namoadr\ with an admissible and consistent MVH on Figure~\ref{fig:example_graph}.}
\label{tab:naive_trace}
\begin{tabular}{@{} >{\centering\arraybackslash}m{0.025\linewidth} | m{0.535\linewidth} | >{\centering\arraybackslash}m{0.39\linewidth} @{}}
\toprule
It. & \centering $\open{}$ candidates $\langle s,\vecv{g},\vecv{h},\vecv{f}\rangle$ & DR action \\\midrule

0 &
\shortstack[l]{
$\rightarrow \langle \sstart,(0,0),(2,14),(2,14)\rangle$
}
&
$\GTcl{\sstart}\leftarrow 0$ \\[0.4ex]
\hline

1 &
\shortstack[l]{
$\rightarrow \langle \sgoal,(3,15),(0,0),(2,14)\rangle$ \\ \hspace{3.5mm}
$\langle s_1,(7,9),(2,7),(9,16)\rangle$ \\ \hspace{3.5mm}
$\langle s_3,(10,6),(1,6),(11,12)\rangle$
}
&
$\GTcl{\sgoal}\leftarrow15$ \\[0.4ex]
\hline

2 &
\shortstack[l]{
$\rightarrow \langle s_1,(7,9),(2,7),(9,16)\rangle$ \\ \hspace{3.5mm}
$\langle s_3,(10,6),(1,6),(11,12)\rangle$
}
&
$F(s_1,(7,9)) \leftarrow F(s_1,(7,9)) \setminus \{(9, 16)\}$ \\[0.4ex]
\hline

3 &
\shortstack[l]{
$\rightarrow \langle s_3,(10,6),(1,6),(11,12)\rangle$ \\ \hspace{3.5mm}
$\langle s_1,(7,9),(8,2),(15,11)\rangle$
}
&
$\GTcl{s_3}\leftarrow6$ \\[0.4ex]
\hline

4 &
\shortstack[l]{
$\rightarrow \langle \sgoal,(12,13),(0,0),(12,13)\rangle$ \\ \hspace{3.5mm}
$\langle s_1,(7,9),(8,2),(15,11)\rangle$ \\ \hspace{3.5mm}
$\langle s_2,(17,7),(1,1),(18,8)\rangle$
}
&
$\GTcl{\sgoal}\leftarrow13$ \\[0.4ex]
\hline

5 &
\shortstack[l]{
$\rightarrow \langle s_1,(7,9),(8,2),(15,11)\rangle$ \\ \hspace{3.5mm}
$\langle s_2,(17,7),(1,1),(18,8)\rangle$
}
&
\shortstack[c]{
node is pruned, \\ $\GTcl{s_1}\prec\Tr(8,10)$
} \\

\bottomrule
\end{tabular}
\end{table}\vspace{0.5\baselineskip}

Table~\ref{tab:naive_trace} traces the execution. A node trace is denoted as $\langle s, \vecv{g}, \vecv{h}, \vecv{f}_{\min \text{lex}} \rangle$. 
\begin{itemize}
    \item \textbf{Iteration 0:} Extracts $\langle \sstart, (0,0), (2,14), (2,14)\rangle$. We initialize $\GTcl{\sstart} \leftarrow 0$. Successors ($\sgoal, s_1, s_3$) are generated and inserted into \open{}.
    \item \textbf{Iteration 1:} Extracts $\langle \sgoal, (3,15), (0,0), (2,14)\rangle$. $\GTcl{\sgoal}$ updates to $15$.
    \item \textbf{Iteration 2:} Extracts $\langle s_1, (7,9), (2,7), (9,16)\rangle$. Since $\Tr((9,16)) = 16 \succ 15$ (the value of $\GTcl{\sgoal}$), the $\vecv{f}$-value $(9,16)$ is pruned. The next lexicographically-smallest $\vecv{f} \in F(s_1, (7,9))$ is {$(15,11)$}, created with $\vecv{h} = (8,2)$. The updated node remains in \open{} as $\langle s_1, (7,9), (8,2), (15,11)\rangle$.
    \item \textbf{Iteration 3:} Extracts $\langle s_3, (10,6), (1,6), (11,12)\rangle$. $\GTcl{s_3}$ updates to $6$. Successors ($\sgoal, s_2$) are generated.
    \item \textbf{Iteration 4:} Extracts $\langle \sgoal, (12,13), (0,0), (12,13)\rangle$. $\GTcl{\sgoal}$ updates to $13$.
    \item \textbf{Iteration 5:} Re-extracts the updated node $\langle s_1, (7,9), (8,2), (15,11)\rangle$. It generates a successor~$s_3$ with cost $\vecv{g} = (8,10)$. The algorithm checks if $\Tr(8,10) = 10 \prec 6$ (the value of $\GTcl{s_3}$). As it is not, the algorithm incorrectly assumes the path is dominated and prunes it.
\end{itemize}

At Iteration 5, the true full cost vector of the previously closed $s_3$ node is $(10,6)$. The new $\vecv{g}$-value is $(8,10)$. Clearly, $(10,6)$ does \emph{not} weakly dominate $(8,10)$. The pruning is an artifact of the truncated comparison, which assumes $g_1 \ge 10$. However, the new node's $g_1 = 8 < 10$. This proves that lexicographic extraction of $\vecv{f}$-values under an MVH does not guarantee lexicographic expansion of $\vecv{g}$-values, breaking Property~\ref{prop:dr}.
\end{proof}

Because the $t$-discarding framework relies entirely on the invariant established by Property~\ref{prop:dr}, applying it without modification under an MVH causes the search to incorrectly prune Pareto-optimal paths (as demonstrated in Iteration~5), breaking completeness and optimality. Preserving correctness therefore requires re-establishing suitable ordering invariants, which motivates the design of the safe search frameworks presented in the following sections.

\ignore{
Because the $t$-discarding framework relies on the invariant established by Property~\ref{prop:dr}, applying it without modification under an MVH causes the search to prune non-dominated paths, breaking completeness and optimality.

\begin{theorem}
\oldlabel{old:thm:badexample}
Naively applying dimensionality reduction within the \namoadr\ framework when using a multi-valued heuristic does not yield a sound algorithm.
\end{theorem}
\begin{proof}
Follows directly from Lemma~\ref{lem:mvh-breaks-prop1}. As demonstrated in Iteration 5 of the execution trace (Table~\ref{tab:naive_trace}), the violation of lexicographic $\vecv{g}$-value ordering causes the algorithm to incorrectly $t$-discard the Pareto-optimal path $\pi_4 = (\sstart, s_1, s_3, s_2, \sgoal)$ (which passes through $s_3$ with cost $(8,10)$), rendering the search incomplete.
\end{proof}

This incompatibility establishes that preserving correctness requires re-establishing suitable ordering invariants when utilizing MVH alternatives. We therefore turn to the design of search frameworks that enable safe DR while maintaining soundness under MVHs.
}
\ignore{
\OS{OLD TEXT}

In this section we show that naively combining dimensionality reduction with multi-valued heuristics (MVHs) may lead to unsound pruning, by naively attempting to use MVHs within \namoadr framework.

Previously, we showed that to perform dimensionality reduction, \namoadr\ relies on a key ordering invariant \emph{monotonic growth} (Section~\ref{notation:monotonic-growth}) induced by the combination of a consistent SVH and lexicographic extraction of nodes by $\vecv f$-values from \open{}. We also showed that nodes are extracted from \open{} in lexicographically non-decreasing order, and that because \namoadr\ is limited to consistent SVH, nodes of the same state are extracted in lexicographically non-decreasing order of $\vecv g$-values. However, the last logic does not apply for MVHs, since $\vecv h$-values are no longer the same among nodes of the same states.

Generally speaking, by allowing two nodes $n, n'$ of the same state to have different $\vecv h$-values, then it is also possible that $\vecv{g}(n')\lexprec \vecv{g}(n)$ while $\vecv{f}(n)\lexsucc\vecv{f}(n')$. 
When extracting nodes by $\vecv f$-values, nodes with larger first-objective $g_1$ costs may be extracted before smaller ones.


To demonstrate the problem, we will use \namoa's presentation for \open{} node entries $n=(s,\vecv{g}, F(s,\vecv{g}))$, where $F(s,\vecv{g}) = \nd\{\vecv{g}+\vecv{h}\mid \vecv{h}\in H(s)\}$. The node selected for extraction from \open{} is the node with lexicographically-smallest $\vecv f$-value among all of currently $\vecv f \in F$ in \open{}. Like \namoadr, to perform dominance checks with DR, we will maintain for each state $s$ a \emph{truncated} closed-group of nodes $\GTcl{s}$.

\input{figures/example_graph.tex}

Figure~\ref{fig:example_graph} shows a bi-objective search graph. Given the query $\sstart \to \sgoal$, the Pareto frontier consists of four solution paths:
$\pi_1=(\sstart,\sgoal)$ with $\mathbf{c}(\pi_1)=(3,15)$,
$\pi_2=(\sstart,s_3,\sgoal)$ with $\mathbf{c}(\pi_2)=(12,13)$,
$\pi_3=(\sstart,s_1,s_3,s_2,\sgoal)$ with $\mathbf{c}(\pi_3)=(16,12)$, and
$\pi_4=(\sstart,s_3,s_2,\sgoal)$ with $\mathbf{c}(\pi_4)=(18,8)$.

\input{tables/simple_example.tex}

Table~\ref{tab:simple_example} specifies an admissible and consistent multi-valued heuristic $H$ for the query $\sstart \to \sgoal$, along with the corresponding SVH, given as the ideal-point heuristic $\hip{}$
For every state $s$, $H(s)$ gives the exact Pareto frontier~$\Pi^*$ of cost vectors from $s$ to $\sgoal$.


\vspace{0.5\baselineskip}
\begin{table}[t]
\centering
\footnotesize
\setlength{\tabcolsep}{2.5pt}
\renewcommand{\arraystretch}{0.97}
\caption{Execution trace of naively combining \namoadr\ with an admissible and consistent MVH on Figure~\ref{fig:example_graph}.}
\label{tab:naive_trace}
\begin{tabular}{@{} >{\centering\arraybackslash}m{0.025\linewidth} | m{0.535\linewidth} | >{\centering\arraybackslash}m{0.39\linewidth} @{}}
\toprule
It. & \centering $\open{}$ candidates $\langle s,\vecv{g},\vecv{h},\vecv{f}\rangle$ & DR action \\\midrule

0 &
\shortstack[l]{
$\rightarrow \langle \sstart,(0,0),(2,14),(2,14)\rangle$
}
&
$\GTcl{\sstart}\leftarrow 0$ \\[0.4ex]
\hline

1 &
\shortstack[l]{
$\rightarrow \langle \sgoal,(3,15),(0,0),(2,14)\rangle$ \\ \hspace{3.5mm}
$\langle s_1,(7,9),(2,7),(9,16)\rangle$ \\ \hspace{3.5mm}
$\langle s_3,(10,6),(1,6),(11,12)\rangle$
}
&
$\GTcl{\sgoal}\leftarrow15$ \\[0.4ex]
\hline

2 &
\shortstack[l]{
$\rightarrow \langle s_1,(7,9),(2,7),(9,16)\rangle$ \\ \hspace{3.5mm}
$\langle s_3,(10,6),(1,6),(11,12)\rangle$
}
&
$F(s_1,(7,9)) \leftarrow F(s_1,(7,9)) \setminus \{(9, 16)\}$ \\[0.4ex]
\hline

3 &
\shortstack[l]{
$\rightarrow \langle s_3,(10,6),(1,6),(11,12)\rangle$ \\ \hspace{3.5mm}
$\langle s_1,(7,9),(8,2),(15,11)\rangle$
}
&
$\GTcl{s_3}\leftarrow6$ \\[0.4ex]
\hline

4 &
\shortstack[l]{
$\rightarrow \langle \sgoal,(12,13),(0,0),(12,13)\rangle$ \\ \hspace{3.5mm}
$\langle s_1,(7,9),(8,2),(15,11)\rangle$ \\ \hspace{3.5mm}
$\langle s_2,(17,7),(1,1),(18,8)\rangle$
}
&
$\GTcl{\sgoal}\leftarrow13$ \\[0.4ex]
\hline

5 &
\shortstack[l]{
$\rightarrow \langle s_1,(7,9),(8,2),(15,11)\rangle$ \\ \hspace{3.5mm}
$\langle s_2,(17,7),(1,1),(18,8)\rangle$
}
&
\shortstack[c]{
node is pruned, \\ $\GTcl{s_1}\prec\Tr(8,10)$
} \\

\bottomrule
\end{tabular}
\end{table}\vspace{0.5\baselineskip}

Table~\ref{tab:naive_trace} illustrates the attempt of naively combining \namoadr\ with $H$. Each iteration represents an extraction of a node. The table shows a trace of \open{} with an arrow that signs at the node chosen for extraction. Beside it, a description of the taken action at the current iteration.

For convenience, instead of tracing the full $F$ group, we write for every node only its lexicographically smallest $\vecv f\in F$. We also added to the trace of each node the $\vecv h$ component that made its lexicographically-smallest $\vecv f$. To conclude, a trace of a node in the table is given as $\langle s, \vecv g, \vecv h, \min_{lex} \{ \vecv f | {\vecv{f}\in F(s, \vecv g)} \} \rangle$, where $\vecv h + \vecv g= \vecv f_{\min lex}$.

In addition, because this is a bi-objective graph, the truncated vectors are of size 1. This allows us to simplify the demonstration and to treat each closed-set $\GTcl{s}$ as a scalar instead of a set.

\paragraph{Trace overview.}
\begin{itemize}

\item Iteration~0 extracts node $\langle \sstart,(0,0),(3,15),(3,15)\rangle$.  
$\sstart$ closed-set $\GTcl{\sstart}=\infty$ is updated with $\GTcl{\sstart}\leftarrow0$, since $\Tr((0,0)) = 0 \prec \infty$.
Successor nodes with states $\sgoal$, $s_1$, and $s_3$ are generated and inserted into $\open{}$.

\item Iteration~1 extracts node $\langle \sgoal,(3,15),(3,15),(3,15)\rangle$ that was generated in iteration~0.
$\sgoal$ closed-set $\GTcl{\sgoal}=\infty$ is updated with $\GTcl{\sgoal}\leftarrow15$, since $\Tr((3, 15)) = 15 \prec \infty$.

\item Iteration~2 extracts node $\langle s_1,(7,9),(3,8),(10,17)\rangle$ that was generated in iteration~0.
Since $\Tr((10,17))=17$ and $\GTcl{\sgoal}=15$, $(10,17)$ is removed from $F(s_1, (7, 9))$ due to pruning. The next lexicographically-smallest $\vecv f\in F(s_1, (7,9))$ is $\vecv f$-value $(16, 12)$ that was created with $\vecv h = (9,3)$. The node remains in \open{} and now has a trace $\langle s_1,(7,9),(9,3),(16,12)\rangle$.

\item Iteration~3 extracts node $\langle s_3,(10,6),(2,7),(12, 13)\rangle$ that was generated in iteration~0. $s_3$ closed-set $\GTcl{s_3}=\{\infty\}$ is updated with $\GTcl{s_3}\leftarrow6$, since $\Tr((10, 6)) = 6 \prec \infty$.
Successor nodes with states $\sgoal$ and $s_2$ are generated and inserted into $\open{}$.

\item Iteration~4 extracts node $\langle \sgoal,(12,13),(0,0),(12,13)\rangle$ that was generated in iteration~3.
$\sgoal$ closed set $\GTcl{\sgoal}=15$ is updated with $\GTcl{\sgoal}\leftarrow13$, since $\Tr((12, 13)) = 13 \prec 15$.

\item Iteration~5 extracts node $\langle s_1,(7,9),(9,3),(16,12)\rangle$ that was last updated in iteration~2.
When attempting to generate a new node $s_3$ with $\vecv{g}=(8,10)$, it holds that $\Tr(8,10) = 10 \prec \GTcl{s_3}=6$. Therefore, the node is discarded due to pruning.

\end{itemize}

In the last iteration, $\langle s_1,(7,9),(9,3),(16,12)\rangle$ is discarded because $\Tr(8,10) = 10 \prec \GTcl{s_3}=6$. However, the vector $(8,10)$ cannot actually be $t$-discarded by the content of the original full vectors that make the closed-set $\GTcl{s_3}$. If the closed-set wasn't truncated, then $\mathbf{G}_{\mathbf{cl}}(s_3) = (10, 6)$. However, $(10, 6) \not \prec (8,10)$. 


\begin{theorem}
\label{thm:badexample}
Naively applying dimensionality reduction within \namoadr\ when using a multi-valued heuristic generally does not yield a sound algorithm.
\end{theorem}
\begin{proof}[\upshape\textbf{Proof}]
The trace in Table~\ref{tab:naive_trace} provides a counter example.
\end{proof}

We finished demonstrating that preserving correctness requires re-establishing suitable ordering invariants in the presence of MVHs alternatives. We therefore turn to the design of search frameworks that enable dimensionality reduction while maintaining soundness under multi-valued heuristics.
}

\section{\soundmvh}
\label{sec:namoadr_mvh}

As proven in Sec.~\ref{sec:incompatibility}, the variability of heuristic estimates within an MVH destroys the ordering invariant required for DR. However, if we group nodes not merely by their state $s$, but by the specific state-heuristic pair $(s, \vecv{h})$, we can theoretically isolate the variance and restore the invariant. 

To achieve this, we can force an MVH to behave locally like an SVH. When a node $n$ expands state $s$, instead of evaluating the successor $s'$ with the entire set $H(s')$, the algorithm must generate a \emph{distinct} successor node for every valid $\vecv{h}' \in H(s')$. By assigning a separate search node for each heuristic value, the search space explodes into a Cartesian product of all non-dominated $\vecv{g}$-values and all available $\vecv{h}$-values for each state. 

Unfortunately, splitting the nodes is not enough and we must also ensure that the sequence of heuristic estimates chosen along a path preserves the monotonic growth of evaluation vectors. To recover the monotonic growth required for DR, the search must enforce \emph{path-consistent heuristic selection} (\soundmvhdefinitionnote). This means that when a node $n$ generates a successor $n'$, it can only bind to an $\vecv{h}' \in H(s')$ that satisfies the consistency triangle inequality: $\vecv{h}(n) \preceq \vecv{c}(s,s') + \vecv{h}'$.

If the MVH is consistent, and the algorithm enforces this path-consistent selection, we can formally guarantee that $t$-discarding becomes safe again.

\begin{theorem}[Restored $\vecv{g}$-value ordering]
\label{thm:restored-g-order}
Assume $H$ is a consistent MVH and the algorithm employs path-consistent heuristic selection. Let $n$ and $n'$ be two extracted nodes that share the same state $s$ and the exact same selected heuristic vector $\vecv{h}$. If $n$ is extracted before $n'$, then $\vecv{g}(n) \le_{\rm lex} \vecv{g}(n')$.
\end{theorem}
\emph{\soundmvhproofnote}

\paragraph{Limitations of \soundmvh.}
Thm.~\ref{thm:restored-g-order} 
implies that $t$-discarding remains sound if the algorithm maintains a separate local truncated closed-set $G_{\mathrm{cl}}^{Tr}(s, \vecv{h})$ for each valid $\vecv{h} \in H(s)$ 
(as shown in the red annotations of Algorithm~\ref{alg:soundmvh}). However, while theoretically sound, this framework is practically crippled by three major limitations:
\begin{enumerate}
    \item[\textbf{L1}] \textbf{State-Space Explosion:} As established, generating a successor for every consistent $\vecv{h}'$ forces the algorithm to generate an overwhelming number of redundant nodes representing the exact same physical path, severely cluttering \open{}.
    \item[\textbf{L2}] \textbf{Fragmentation of Pruning Power:} By partitioning the closed-sets by $\vecv{h}$, the algorithm performs dominance checks against much smaller sets, drastically reducing the likelihood of successful $t$-discarding.
    \item[\textbf{L3}] \textbf{The Consistency Burden:} Constructing tight, strictly consistent MVHs is both difficult and computationally expensive (see also Sec.~\ref{sec:consistency-fixed}).
\end{enumerate}

We therefore seek an alternative framework that avoids the Cartesian product explosion entirely, preserves the efficiency of sound DR, and only requires \emph{admissible} MVHs.

\section{\frankenmvh: Lazy DR}
\label{sec:frankenmvh}

We now present our primary algorithmic contribution, \frankenmvh\ (where the ``L'' explicitly stands for \emph{Lazy}). This framework enables the sound use of DR together with MVHs while dropping the consistency requirement. It requires only the \emph{admissibility} of the MVH. 

The key idea is to use DR optimistically: assume the required ordering invariants hold, dynamically detect when they are violated, and fall back to a full Pareto dominance test only when absolutely necessary. Similar to \namoadr, the algorithm maintains a priority queue \open{} sorted by lexicographically-smallest $\vecv{f}$-values. However, it maintains \emph{two} closed-sets for every state $s$: (i) a truncated closed-set $\GTcl{s}$ for fast $t$-discarding, and (ii) a full non-dominated set $G_{\mathrm{cl}}(s)$ of discovered $\vecv{g}$-values, serving as a correctness-critical fallback.

\begin{algorithm}[t!]
  \caption{\frankenmvh{}}
  \label{alg:cfs}
  \begin{algorithmic}[1]
    \Statex \textbf{Input:} A MOS instance $(S,E,\vecv c,\sstart,\sgoal)$;
    \Statex \hspace{10mm} an admissible MVH $H$.
    \Statex \textbf{Output:} A cost-unique Pareto-optimal solution set $\text{Sols}$.

    \vspace{1mm}

    \State $\text{Sols}\gets\emptyset$
    \ForAll{$s\in S$}
      \State ${\GTcl{s}\gets \emptyset}; G_{\mathrm{cl}}(s)\gets\emptyset$
    \EndFor
    
    \State $n\gets$ new node with $s(n)=\sstart$
    \State $\text{parent}(n)\gets \text{NULL}$; 
    $\vecv g(n)\gets \vecv 0$; 
            $\vecv f(n)\gets H_{lex}^{\min}(\sstart)$ 

    \State $\open \leftarrow \{n\}$
    \While{$\open\ \neq \emptyset$}
    \State $n \leftarrow \open{}.\text{Pop}$
    \If{$\exists\, \vecv f_{\rm goal} \in \GTcl{\sgoal}~s.t.~
        \vecv f_{\rm goal} \preceq \Tr(\vecv{f}(n))$} \label{line:alg2-reinsertion}
    \State $\vecv{h}' \leftarrow
      \textsc{ChooseH}(s(n), \vecv g(n),\GTcl{\sgoal})$ \label{line:alg2-chooseh-1}
    \If{$h' \neq \bot$} \label{alg2:chooseh-bot} 
    \State $\vecv h(n) \leftarrow \vecv h'$;
    \State $\vecv{f}(n) \leftarrow \vecv{g}(n) + \vecv{h}'$
    \State $\open.\text{Insert}(n)$
    \EndIf     
    \Continue
    \EndIf
    \If{\Call{LocalDomCheck}{$s(n), \vecv{g}(n)$}} \label{line:alg2-local-dom}
    \Continue
    \EndIf
    \State $\GTcl{s(n)} \leftarrow
      \!\left(\GTcl{s(n)} \cup \{\Tr({\vecv{g}(n)})\}\right)$;
    \State $G_\text{cl}(s(n)) \leftarrow
      G_\text{cl}(s(n)) \cup \{\vecv{g}(n)\}$
    \If{$s(n) = \sgoal$} \label{line:alg2-sols-update}
    \State $\sols \leftarrow \sols \cup \{n\}$
    \Continue
    \EndIf
    \ForAll{$s' \in \mathrm{Succ}(s)$}
    \State $n'\gets$ a new node with $s(n')=s'$ \label{line:alg2-new-successor}
    \State $\mathrm{parent}(n')\gets n$; 
           $\vecv g(n')\gets \vecv g(n)+\vecv c(s(n),s')$
    \State $\vecv{h}' \leftarrow
      \textsc{ChooseH}(s', \vecv{g}(n'), \GTcl{\sgoal})$ \label{line:alg2-chooseh-2}
    \If{$\vecv{h}' = \bot$} \label{line:alg2-bot-2}
    \Continue
    \EndIf
    \If{\Call{LocalDomCheck}{$s', \vecv{g}(n')$}} \label{line:alg2-local-pruning}
    \Continue
    \EndIf
    \State $\vecv f(n')\gets \vecv g(n')+ \vecv h'$
    \State $\open.\text{Insert}(n')$
    \EndFor
    \EndWhile
    \State \Return $\sols$
    \vspace{-2mm}
    \Statex
    \Function{ChooseH}{$s,\vecv g, T$} \label{line:alg2-choose-h}
    \ForAll{$\vecv h \in H(s)$ in lexicographic order}
        \State $\vecv f \gets \vecv g+\vecv h$
        \If{$T \preceq \Tr(\vecv f)$} \Comment{$t$-discarding}
            \State \textbf{continue}
        \EndIf
        \State \Return $\vecv h$
    \EndFor
    \State \Return $\bot$
\EndFunction
    \vspace{-2mm}
    \Statex
    \Function{LocalDomCheck}{$s,\vecv g$} \label{line:alg2-local-dom-check-def}
  \If{$\exists\,\vecv g' \in \GTcl{s}:
      \Tr(\vecv g') \preceq \Tr(\vecv g)$} \label{line:alg2-t-discarding}
    \If{$\max_{\vecv g' \in \GTcl{s}} g'_1 \le g_1$} \Comment{$t$-discarding} \label{line:alg2-t-discarding-violation}
      \State \Return \textbf{true}
    \EndIf
    \If{$\exists\,\vecv g' \in G_{\mathrm{cl}}(s): \vecv g' \preceq \vecv g$} \label{line:alg2-fallback}
      \State \Return \textbf{true}
    \EndIf
  \EndIf
  \State \Return \textbf{false}
\EndFunction
  \end{algorithmic}
\end{algorithm}

\subsection{Algorithmic Mechanics: The Node Lifecycle}
To understand how \frankenmvh\, outlined in Alg.~\ref{alg:cfs}, overcomes the state-space explosion (Sec.~\ref{sec:namoadr_mvh},~\textbf{L1}), we trace the lifecycle of a node as it is generated, extracted, and potentially re-evaluated.

\paragraph{1. Generation \& Lazy Binding (\textsc{ChooseH}).}
Instead of eager node-splitting (which creates a Cartesian product of $\vecv{g}$-values and $\vecv{h}$-values), \frankenmvh\ generates \emph{one} search node $n=(s, \vecv{g}, \vecv{f})$ per physical path. When a successor is generated, it is evaluated \emph{lazily}. 
The \textsc{ChooseH} function (Line~\ref{line:alg2-choose-h}) scans $H(s)$ and binds the path to the lexicographically-smallest $\vecv{h}$-value that is not already $t$-discarded by the current known solutions. This collapses the combinatorial explosion back to a single node in \open{}.

\paragraph{2. Extraction \& Lazy Re-evaluation.}
If an extracted node is later found to be globally dominated by a newly discovered solution, it is not immediately discarded. Instead, the algorithm performs a \emph{lazy re-evaluation} (Line~\ref{line:alg2-reinsertion}): it calls \textsc{ChooseH} again to find the \emph{next} best non-dominated $\vecv{h}$-value and re-inserts the updated node into \open{}. This simulates exploring all valid heuristics without inserting multiple copies of the same path into the queue.

\paragraph{3. Optimistic DR \& Fallback (\textsc{LocalDomCheck}).}
As \frankenmvh\ does not require a consistent MVH, the first objective is not guaranteed to grow monotonically. Lexicographic extraction of $\vecv{f}$-values might expand nodes out of $\vecv{g}$-value order, making naive $t$-discarding unsafe.

To resolve this, the \textsc{LocalDomCheck} function (Line~\ref{line:alg2-local-dom-check-def}) first optimistically checks if $\Tr(\vecv{g})$ is dominated by $\GTcl{s}$. If it is, the algorithm validates the invariant: it checks whether the new node's $g_1$ is greater than or equal to the maximum $g_1$ seen so far for state $s$. If $g_1 \ge \max g'_1$, the monotonic growth assumption holds locally, $t$-discarding is sound, and the node is safely pruned. However, if $g_1 < \max g'_1$, a sequence violation is detected. The algorithm falls back to a standard, full Pareto dominance check against~$G_{\mathrm{cl}}(s)$, pruning the node only if it is truly dominated.

\subsection{Theoretical Guarantees}
Correctness of \frankenmvh\ relies on the soundness of these lazy evaluations and fallback mechanisms. As \textsc{ChooseH} only discards heuristic evaluations that are globally pruned by known solutions, and \textsc{LocalDomCheck} detects local ordering violations using full dominance checks, the algorithm never incorrectly prunes a Pareto-optimal path.

\begin{theorem}[Correctness of \frankenmvh]
\label{thm:frankenmvh-correctness}
If $H$ is an admissible MVH then \frankenmvh\ returns the complete Pareto-optimal solution set $\Pi^*$.
\end{theorem}
\emph{\frankenmvhproofnote}

While inconsistent heuristics may trigger repeated full Pareto fallbacks (adding subsequently dominated cost vectors to $G_{\mathrm{cl}}(s)$), we demonstrate in Sec.~\ref{sec:evaluation} that these effects remain limited in practice, allowing \frankenmvh\ to achieve dramatic speedups when the MVH provides strong search guidance.

\ignore{

\section{\frankenmvh}
\oldlabel{old:sec:frankenmvh}

We now present a second framework that enables the sound use of DR together with MVHs. Crucially, this framework requires only the \emph{admissibility} of the heuristic, dropping the consistency requirement. The key idea is to use DR optimistically: assume the required ordering invariants hold, detect the exact moments they are violated, and dynamically fall back to a full Pareto dominance test only when absolutely necessary.

Similar to \namoadr, \frankenmvh\ (Alg~\ref{alg:cfs}) is a best-first MOS algorithm that combines a \namoa-style node-selection policy with DR-based pruning. Like \namoa\ and \namoadr, the algorithm maintains a priority queue \open{} and selects for extraction the node with the lexicographically-smallest $\vecv{f}$-value. However, to support its optimistic DR strategy, it maintains \emph{two} closed-sets for every state $s$: 
(i) a truncated closed-set $\GTcl{s}$ used for fast pruning based on $t$-discarding, and 
(ii) a full non-dominated set $G_{\mathrm{cl}}(s)$ of discovered $\vecv{g}$-values, which serves as a correctness-critical fallback. Lastly, $\sols$ is used to track solutions.

\subsection{Algorithmic Mechanics}

The core mechanics of \frankenmvh\ are designed to overcome the limitations of the strict-consistency approach discussed in Sec.~\ref{sec:namoadr_mvh}.

\paragraph{Tackling the State-Space Explosion: Lazy Evaluation.}
In an attempt to avoid generating redundant nodes for every possible consistent heuristic combination, \frankenmvh\ represents each node simply as $n=(s, \vecv{g}, \vecv{f})$. A node is instantiated with only the single $\vecv{h}$-value that yields the lexicographically-smallest globally non-dominated $\vecv{f}$-value. This is handled by the \textsc{ChooseH} function (Line~\ref{line:alg2-choose-h}). 

\textsc{ChooseH}$(s,\vecv{g},T)$ scans $H(s)$ in lexicographic order of $\vecv{g}+\vecv{h}$ and returns the first heuristic vector that is not already $t$-discarded by the target set $T$ (typically $\GTcl{\sgoal}$). It returns $\bot$ only if all admissible heuristic choices are globally discarded. 

If an extracted node is later found to be globally dominated by a newly discovered solution, it is not immediately discarded. Instead, the algorithm performs a \emph{lazy evaluation} (Line~\ref{line:alg2-reinsertion}): it calls \textsc{ChooseH} again to find the next lexicographically-smallest non-dominated $\vecv{f}$-value and re-inserts the node into \open{}. This strategy preserves best-first extraction while avoiding the simultaneous insertion of multiple nodes representing the same path.

\paragraph{Tackling the Consistency Burden: Optimistic Local Checks.}
Because \frankenmvh\ does not require a consistent MVH, the monotonic growth invariant (Property~\ref{prop:dr}) is not guaranteed. Lexicographic extraction of $\vecv{f}$-values might occasionally result in nodes of the same state being expanded out of lexicographic $\vecv{g}$-value order. This implies that naive $t$-discarding can't be used.

To resolve this, the \textsc{LocalDomCheck} function (Line~\ref{line:alg2-local-dom-check-def}) performs local dominance checking in two sequential stages:
\begin{enumerate}
    \item \textbf{Optimistic $t$-discarding:} First, it checks if $\Tr(\vecv{g})$ is dominated by $\GTcl{s}$ (Line~\ref{line:alg2-t-discarding}). If it is not, the node is immediately accepted. 
    \item \textbf{Invariant Validation \& Fallback:} If the node is (allegedly) $t$-discarded, the algorithm must verify if the $t$-discarding is actually sound. It checks whether the new node $g_1$ is greater than or equal to the maximum $g_1$ seen so far for state $s$, namely $\max g'_1$ (Line~\ref{line:alg2-t-discarding-violation}). 
    \begin{itemize}
        \item If $g_1 \ge \max g'_1$, then Property~\ref{prop:dr} holds and the monotonic growth assumption is locally valid. Thus, $t$-discarding is sound, and the node is safely pruned.
        \item If $g_1 < \max g'_1$, the algorithm has detected a violation of the lexicographic extraction order. The $t$-discarding is potentially unsafe and the algorithm falls back to a full Pareto dominance check against $G_{\mathrm{cl}}(s)$ and prunes the node only if it is dominated (Line~\ref{line:alg2-fallback}).
    \end{itemize}
\end{enumerate}

\subsection{Theoretical Guarantees}

The correctness of \frankenmvh\ relies on proving that these lazy evaluations and fallback mechanisms are completely sound.

\begin{lemma}[Soundness of \textsc{ChooseH}]
\oldlabel{old:lem:chooseh-sound}
Let a generated node be identified by $n=(s, \vecv{g}, \vecv{f})$. If \textsc{ChooseH}$(s,\vecv{g}, \GTcl{\sgoal})$ returns $\bot$, then for every heuristic vector $\vecv{h}\in H(s)$, the resulting evaluation vector $\vecv{f} = (\vecv{g}+\vecv{h})$ is $t$-discarded by $\GTcl{\sgoal}$. Hence, no completion of $n$ can yield a Pareto-optimal solution.
\end{lemma}
\begin{proof}
\textsc{ChooseH} scans all heuristic vectors $\vecv{h}\in H(s)$ and returns the lexicographically-first vector which is not $t$-discarded by $T$. When calling \textsc{ChooseH} against~$\GTcl{\sgoal}$, returning $\bot$ means that every admissible heuristic choice is discarded by an existing solution in~$\GTcl{\sgoal}$. Since $H$ is admissible, each $\vecv{f} = \vecv{g}+\vecv{h}$ is a lower bound on any completion through $s$. Therefore, if all such evaluations are dominated by already discovered solution costs, no continuation of $\vecv{g}$ can produce a new Pareto-optimal solution.
\end{proof}


\begin{lemma}[Soundness of \textsc{LocalDomCheck}]
\oldlabel{old:lem:franken-local-full}
The two stages of pruning executed within \textsc{LocalDomCheck} are sound.
\end{lemma}
\begin{proof}
Let a generated node be identified by $n=(s, \vecv{g}, \vecv{f})$. If $\GTcl{s}\not\prec\Tr(\vecv{g})$, then $G_{\mathrm{cl}}(s) \not\prec \vecv{g}$ must also be true, and the node is correctly not pruned (line~\ref{line:alg2-t-discarding}). If $\GTcl{s} \preceq \Tr(\vecv{g})$, we divide the proof into cases: 
If $\max\{g'_1 \mid \vecv{g}' \in \GTcl{s}\} \le g_1$, then every previously extracted node with state $s$ has a first-component cost of at most $g_1$, which is the exact condition required for $t$-discarding correctness.
If $\max\{g'_1 \mid \vecv{g}' \in \GTcl{s}\} > g_1$ then $g$ is not dominated by discovered vectors in $\GTcl{s}$, and pruning is evaluated against the full discovered Pareto frontier discovered in the closed-set $G_{\mathrm{cl}}(s)$ (line~\ref{line:alg2-fallback}). In this final case, pruning is performed exactly as in the standard, proven \namoa\ dominance check.
\end{proof}

\begin{theorem}[Correctness of \frankenmvh{}]
\oldlabel{old:thm:franken-correct}
Assume $H$ is an admissible multi-valued heuristic. Then \frankenmvh\ returns the complete Pareto-optimal solution set.
\end{theorem}
\begin{proof}
A generated node $n=(s, \vecv{g}, \vecv{f})$ may be discarded in only two ways. First, it may fail the global check. By Lemma~\ref{lem:chooseh-sound}, this can happen only after all heuristic choices have been exhausted, in which case no completion of $\vecv{g}$ can lead to a Pareto-optimal solution. Second, it may be removed by the local check, which is proven sound by Lemma~\ref{lem:franken-local-full}. Thus, every discarded node is truly dominated, while every non-dominated path remains eligible for extraction. Since the search is best-first and $H$ is admissible, every Pareto-optimal solution is eventually generated and inserted into $\sols$. Therefore, the returned set is exactly the Pareto-optimal solution set $\Pi^*$.
\end{proof}

\frankenmvh\ combines the efficiency of DR with the advantageous guidance of MVHs. By detecting violations in the $\vecv{g}$-value ordering required for $t$-discarding, the algorithm preserves exact correctness without requiring heuristic consistency. While some heuristics may lead to repeated activations of the full Pareto dominance fallback (and consequently, subsequently discovered nodes may be added to $G_{\mathrm{cl}}(s)$), these effects remain highly limited in practice, as we will demonstrate in Sec.~\ref{sec:evaluation}.

\ignore{
\OS{old text}

We now present a second framework that enables sound use of dimensionality reduction with multi-valued heuristics without requiring heuristic consistency, only admissibility.
The key idea is to use dimensionality reduction optimistically while detecting situations in which its ordering assumptions may be violated and falling back to a full Pareto dominance test when necessary.

\frankenmvh\ is a best-first multi-objective search algorithm that uses a \namoa-style node-selection policy, together with DR-based pruning.
The algorithm assumes only admissibility of the multi-valued heuristic.

Like \namoa\ and \namoadr, \frankenmvh\ is a best-first MOS algorithm, which selects for extraction from \open{} a node with the lexicographically-smallest $\vecv f$-value.
In addition, two closed-sets are maintained for every state $s$: (i) a truncated closed-set $\GTcl{s}$ used for fast pruning based on $t$-discarding (DR), and (ii) a full non-dominated set $G_{\mathrm{cl}}(s)$ of discovered $g$-values. The last closed-set is used for correctness-critical fallback checks. Lastly, $\sols$ is used to track solutions.

In \frankenmvh\, each node is represented by $n=(s, \vecv g, \vecv f)$, where $\vecv f$ depends on a selected $\vecv h\in H(s)$.
To avoid redundant node generation, nodes are generated with the $\vecv h$-value that yields the lexicographically-smallest globally non-dominated $\vecv f$-value using \textsc{ChooseH} function (Line~\ref{line:alg2-choose-h}).
If the node is later pruned by an existing solution in $\sols$, the node may be reinserted by computing its next lexicographically-smallest non-dominated $\vecv f$-value (Line~\ref{line:alg2-reinsertion}).
This strategy preserves best-first extraction while avoiding simultaneous insertion of multiple nodes representing the same path, simplifying the usage of \namoa{}s $F(s, \vecv g)$ group of different non-nominated $\vecv f$-values.

\medskip
\noindent\textbf{\textsc{ChooseH}$(s,\vecv g,T)$} (Line~\ref{line:alg2-choose-h})
scans $H(s)$ in lexicographic order of $\vecv g+\vecv h$ and returns the first heuristic vector that is not $t$-discarded by $T$.
It returns $\bot$ if all heuristic choices are discarded.

After extraction from $\open{}$, the node $\vecv f$-value is first tested for global dominance-check against $\GTcl{\sgoal}$.
If $\Tr(\vecv f)$ is dominated, the algorithm attempts to select an alternative $\vecv h$-value using \textsc{ChooseH} and re-insert the node. Only if no such $\vecv h$-value exists, the node is discarded.

\begin{lemma}[Soundness of \textsc{ChooseH}]
\oldlabel{old2:lem:chooseh-sound}
Let a generated node be identified by $n=(s, \vecv g, \vecv f)$.
If \textsc{ChooseH}$(s,\vecv g, \GTcl{\sgoal}$ returns $\bot$ (Line~\ref{line:alg2-bot-1}, \ref{line:alg2-bot-2}), then for every heuristic vector $\vecv h\in H(s)$, the $\vecv f$-value $\vecv f = (\vecv g+\vecv h)$ is $t$-discarded by $\GTcl{\sgoal}$.
Hence no completion of $n$ can yield a Pareto-optimal solution.
\end{lemma}

\begin{proof}[\upshape\textbf{Proof}]
\textsc{ChooseH} scans all heuristic vectors in $\vec h\in H(s)$ and returns the lexicographically-first which is not $t$-discarded by $T$.
When calling \textsc{ChooseH} with $(s,\vecv g, \GTcl{\sgoal}$, returning $\bot$ means that every admissible heuristic choice is discarded by an existing solution in $\GTcl{\sgoal}$ (lines~\ref{line:alg2-chooseh-1}, ~\ref{line:alg2-chooseh-2}). 
Since $H$ is admissible, each $\vecv f= \vecv g+\vecv h$ is a lower bound on any completion through $s$.
Therefore, if all such evaluations are dominated by already discovered solution costs, no continuation of $\vecv g$ can produce a new Pareto-optimal solution.
\end{proof}

\medskip
\noindent\textbf{\textsc{LocalDomCheck}$(s,\vecv g,T)$} (Line~\ref{line:alg2-local-dom-check-def}) Performs local dominance check in two stages. Let a node be identified by $n=(s, \vecv g, \vecv f)$.
First, it performs a local dominance check against $\GTcl{s}$, assuming that all original vectors in $\GTcl{s}$ have a first-component cost smaller than $\vecv g_1$ (Line~\ref{line:alg2-t-discarding}). if $\Tr(\vecv g)$ is not dominated by $\GTcl{s}$, the node is accepted immediately.
If the node is (allegedly) $t$-discarded, it tests whether all previously extracted nodes with a state $s$ have first-component cost no greater than $g_1$ (Line~\ref{line:alg2-t-discarding-violation}).
If it holds, $\vecv g$ is correctly $t$-discarded and the node is pruned.
Otherwise, a violation of the lexicographic extraction order of the $\vecv g$-values for state $s$ is detected.
The algorithm then performs a full Pareto dominance check against $G_{\mathrm{cl}}(s)$ and prunes the node only if it is dominated by a value in $G_{\mathrm{cl}}(s)$ (Line~\ref{line:alg2-fallback}).
In practice, this fallback test is triggered rarely (see~\ref{sec:evaluation}).

The correctness of \textsc{LocalDomCheck} follows from the fact that every pruning decision is either justified directly by the DR ordering assumptions or repaired by a full dominance fallback when those assumptions are violated.

\begin{lemma}[Soundness of the DR stage local check]
\oldlabel{old2:lem:franken-local-dr}
Let a generated node be identified by $n=(s, \vecv g, \vecv f)$.
Suppose (i) $\vecv g'\preceq \Tr(\vecv g)$ for some $\vecv g'\in \GTcl{s}$, and additionally (ii) $\max\{g'_1 \mid \vecv g' \in G_{\mathrm{cl}}(s)\} \le g_1$.
Then pruning $\Tr(\vecv g)$ by $\GTcl{s}$ is sound.
\end{lemma}

\begin{proof}[\upshape\textbf{Proof}]
(ii) certifies that every previously extracted node with state $s$ has first-component cost at most $g_1$, required for $t$-discarding correctness. If (i) also holds, then all the conditions for $t$-discarding are satisfied. the soundness of dominance check derived from the correctness $t$-discarding.
Hence $\vecv g$ is dominated by an already extracted node and can be safely pruned.
\end{proof}

\begin{lemma}[Soundness of \textsc{LocalDomCheck}]
\oldlabel{old2:lem:franken-local-full}
The two stages of pruning done with \textsc{LocalDomCheck} are sound.
\end{lemma}

\begin{proof}[\upshape\textbf{Proof}]
Let a generated node be identified by $n=(s, \vecv g, \vecv f)$.
If $\GTcl{s}\not\prec\Tr(\vecv g)$, then also $G_{\mathrm{cl}}(s) \not \prec \vecv g$, and the node isn't pruned (Line~\ref{line:alg2-t-discarding}).
Else, we show by dividing into cases. If $\max\{g'_1 \mid \vecv g' \in \GTcl{s}\} \leq \vecv g_1$, correctness follows from lemma~\ref{lem:franken-local-dr} (Line~\ref{line:alg2-t-discarding-violation}).
If $\max\{g'_1 \mid \vecv g' \in \GTcl{s}\} > \vecv g_1$, pruning is done compared to the full Pareto frontier discovered so far in the closed-set $G_{\mathrm{cl}}(s)$ (Line~\ref{line:alg2-fallback}). In the last case, pruning is performed exactly as in the standard \namoa-style dominance check, where a node is filtered against all other discovered node of its state.
\end{proof}

\begin{theorem}[Correctness of \frankenmvh{}]
\oldlabel{old2:thm:franken-correct}
Assume $H$ is admissible.
Then \frankenmvh\ returns the complete Pareto-optimal solution set.
\end{theorem}

\begin{proof}[\upshape\textbf{Proof}]
A generated node $n=(s, \vecv g, \vecv f)$ may be discarded in only two ways.

First, it may fail the global check (Line~\ref{line:alg2-chooseh-1}).
By Lemma~\ref{lem:chooseh-sound}, it can happen only after all heuristic choices have been exhausted, in which case no completion of $\vecv g$ can lead to a Pareto-optimal solution. Second, it may be removed by the local check, which proven sound by Lemma~\ref{lem:franken-local-full}.

Thus every discarded node is truly dominated, while every non-dominated path remains eligible for extraction.
Since the search is best-first and $H$ is admissible, every Pareto-optimal solution is eventually generated and inserted into $\sols$.
Therefore the returned set is exactly the Pareto-optimal solution set $\Pi^*$.
\end{proof}

\frankenmvh\ combines the efficiency of dimensionality reduction
with the advantage guidance of MVHs.
By detecting violations in the order of $\vecv g$-values of similar states required for $t$-discarding,
the algorithm preserves correctness, in a way that works for any admissible \mvh{}s.

When the heuristic is inconsistent, violations of the lexicographically non-decreasing ordering of $g$-values for each state may occur more frequently, leading to repeated activation of performing dominance checks against the full Pareto frontier.
As a consequence, cost vectors that are later dominated by other subsequently discovered may be added to $G_{\mathrm{cl}}(s)$.
Nevertheless, these effects remain limited in practice, and they will be viewed in section~\ref{sec:evaluation}.

}}

\section{Construction of Multi-Valued Heuristics}
\label{sec:mvh-construction}
In this section we describe an approach to generate admissible MVHs as well as how to modify them into consistent MVHs.
Recent work by \citeauthor{geisser_admissible}~\shortcite{geisser_admissible} investigated the construction of admissible and consistent MVHs in the context of domain-independent planning. They extended several classes of classical planning heuristics to the MOS setting. However, while their work provides theoretical insights, these constructions rely heavily on the structural properties of factored planning domains (such as STRIPS representations and delete-relaxation reasoning) and are not directly applicable to general graph-based MOS problems.

To evaluate our algorithms, we require domain-agnostic MVHs. We achieve this through a pre-processing phase that performs an approximate backward MOS from the goal.

\paragraph{A*pex-MVH (Admissible MVHs).} \label{sec:apex-mvh}
To generate an admissible MVH, we use an adaptation of the \apex{} algorithm~\cite{apex}. \apex{} is an approximate MOS algorithm designed to efficiently compute an $\varepsilon$-approximation of the exact Pareto-optimal solution set. It achieves this by grouping similar paths and representing each group using a single multi-dimensional vector called an \emph{apex}. 

An apex is constructed by taking the component-wise minimum cost across all paths within its respective group. Consequently, an apex is guaranteed to weakly dominate the cost vector of every actual path it represents. By executing \apex{} as a backward search from $\sgoal$ with an $\varepsilon$-approximation factor, we compute a set of mutually non-dominated apexes for every state $s$. Because each apex weakly dominates the true path costs to the goal, this resulting set $H(s)$ serves as an admissible MVH.

\paragraph{Consistency-Fixed MVHs.} \label{sec:consistency-fixed}
Unfortunately, the \apex-MVH approach is not guaranteesd to be consistent. Thus, we apply a \emph{consistency-fix} procedure modeled after Bellman-Ford label updates to create consistent MVHs based on existing admissible MVHs to use in \soundmvh.

Intuitively, as long as there is an inconsistency in a state's heuristic set, this procedure lowers the violating heuristic values to resolve it. Specifically, if there exist states $s, s'$ and a heuristic vector $\vecv{h}' \in H(s')$ such that $\vecv{c}(s,s') + \vecv{h}' \prec \vecv{h}$ for some existing estimate $\vecv{h} \in H(s)$, the estimate $\vecv{h}$ violates consistency. We resolve this by inserting the smaller vector $\vecv{c}(s,s') + \vecv{h}'$ into $H(s)$ and discarding any newly-dominated values in $H(s)$, including the original $\vecv{h}$. 

This refinement propagates backward through the graph until a fixed point is reached and all local inconsistencies are resolved. Because this procedure only lowers bounds that were originally admissible, the resulting MVH remains admissible and becomes, by construction, consistent. However, the procedure tends to result in more complex, less-informed MVH, as we will see in Section~\ref{sec:evaluation}.

\section{Evaluation}
\label{sec:evaluation}

In this section we evaluate our two new algorithms \soundmvh and \frankenmvh, comparing them to the baselines \namoa and \namoadr.
Experiments were run on three different benchmarks (\netm\, \ny\ and \panda\ which have three, three and eight objectives, respectively).\footnote{\url{https://github.com/CRL-Technion/Multi-Objective-Search-Benchmarks.git}.
}
All experiments were conducted on an AWS computing cluster, with 32 GB of memory and runtime limit of 7200s per query instance (except \texttt{netM10} where we used 3600s), and executed in a single-threaded CPU configuration on Intel Xeon processors (2.1–2.4 GHz).
We implemented all algorithms in C++, using a common codebase as much as possible\footnote{\url{https://github.com/CRL-Technion/bridging-mvh-dr}.}.
For each benchmark we compared 50 different search queries. 
We used \namoadr\ where the SVH is the ideal-point heuristic SVH $\hip$~\cite{salzman2025multi}\footnote{
$\hip$ combines a set of $N$ single-objective heuristics $h_1, \ldots, h_N$. 
Here, $h_i:S\rightarrow \mathbb{R}_{\geq 0}$ corresponds to the shortest path from each state according to the $i$'th objective and $\forall s \in S,~\mathbf{\hip}(s) := ( h_1(s), \ldots, h_N(s) )$.
The ideal point heuristic, which is admissible, is easily computed by running $N$ (single-objective) instances of Dijkstra’s algorithm starting from~$\sgoal$ (i.e., one instance for each objective).
}.

We start by outlining the properties of the MVH we construct.
We then continue with a comparison of the different algorithms across the benchmarks and
conclude by pinpointing the source of the efficiency of \frankenmvh, i.e., how effective is the lazy approach we employ.

\paragraph{MVH construction.}
To obtain admissible and consistent MVHs, we follow the approach described in Sec.~\ref{sec:consistency-fixed}.
Specifically, we run \apex using different approximation factors $\varepsilon$ to obtain the so-called ``A*pex-MVH'' 
and then apply the consistency-fixed procedure.\footnote{We selected 4 $\varepsilon$-approximation factors.
For \netm\ we used $\varepsilon\in\{ {0.2}, {0.6}, {1.6}, {2.4}\}$
For \ny\ we used $\varepsilon\in\{ {0.08}, {0.1}, {0.2}, {0.4}\}$; 
For \panda\ we used $\varepsilon\in\{ {0.6}, {1}, {1.5}, {2}\}$.} 
We refer to these MVHs as (A)-type and (C)-type, respectively.
Fig.~\ref{fig:mvh_runtime_table} reports the MVH size and computation time as a function of $\varepsilon$.
As expected, both size and computation time decrease as $\varepsilon$ increases. However, as we will see shortly, the extra running time invested in creating larger, more informed MVHs will allow, when the MVH provides strong search guidance, a larger speedup across multiple queries. 

\paragraph{Algorithms comparison.}
To compare the different algorithms, we start (Fig.~\ref{fig:fig2-runtime-vs-mvh-size}) by plotting the runtime of each algorithm (for \namoa and \frankenmvh we use both (A)-type and (C)-type) as a function of the MVH size.
For \netm, we see that as the MVH size increases, the obtainable speedup also increases, with a clear advantage for the larger, admissible-only (A)-type MVH over the more complex (C)-type.
This suggests that \netm\ is an environment where stronger heuristic guidance can successfully narrow the search space without being overwhelmed by objective correlation or density.
For \ny, there is no speedup. We attribute this to the fact that this benchmark includes highly-correlated objectives~\cite{halle_correlated}. The high correlation implies that the dominance relationship is more frequently satisfied, resulting in sparse Pareto sets that effectively collapse the search space toward a lower-dimensional frontier where pruning using regular SVH (like \namoadr with $\hip$) remains highly efficient. Thus, the additional guidance of an MVH has little room to improve over the already effective SVH-based pruning.
Finally for \panda, we see a slight speedup for smaller MVHs.
We suggest that the topological density of \panda\ generates a Pareto frontier with high cardinality, so that the objective space becomes effectively saturated with non-dominated solutions. This density dilutes the pruning power of any heuristic, including MVHs.
A relatively small MVH or a SVH is able to cover most of the different trade-offs just by exploring the graph.
Together, these results suggest that the benefit of \frankenmvh\ depends on whether the additional MVH guidance can translate into meaningful pruning. This happens clearly in \netm, while in \ny\ and \panda\ the baseline pruning is already competitive, leaving less room for improvement.
\begin{figure}[t]
    \centering

    \begin{subfigure}{\columnwidth}
        \centering
        \includegraphics[width=0.92\linewidth]{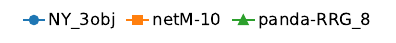}
    \end{subfigure}

    \vspace{0.25em}

    \begin{subfigure}{0.48\columnwidth}
        \centering
        \includegraphics[width=\linewidth]{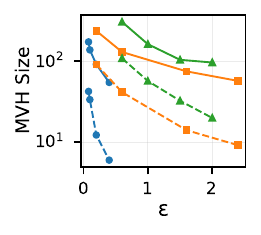}
        \caption{Average MVH size as a function of $\epsilon$.}
        \label{fig:mvh-size}
    \end{subfigure}
    \hfill
    \begin{subfigure}{0.48\columnwidth}
        \centering
        \includegraphics[width=\linewidth]{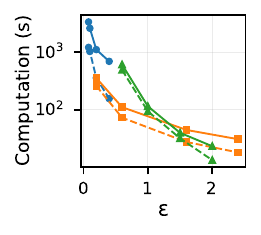}
        \caption{Average computation time as a function of $\epsilon$.}
        \label{fig:mvh-runtime}
    \end{subfigure}

    \caption{MVH size (a) and MVH computation time (b) as a function of $\varepsilon$ for each benchmark.
    Solid and dashed lines correspond to (C)-type and (A)-type MVHs, respectively.}
    \label{fig:mvh_runtime_table}
\end{figure}


\begin{figure}[t]
    \centering

    \includegraphics[width=\columnwidth]{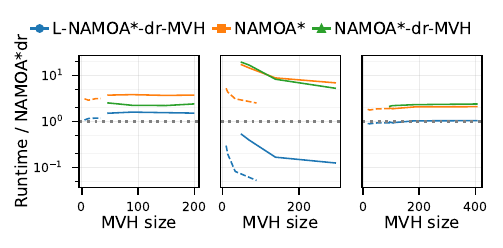}

    \vspace{-1.6em}
    \hspace{0.2em}
    \begin{subfigure}[t]{0.32\columnwidth}
        \centering
        \caption{\ny}
        \label{fig:runtime-mvh-ny}
    \end{subfigure}
    \begin{subfigure}[t]{0.32\columnwidth}
        \centering
        \caption{\netm}
        \label{fig:runtime-mvh-netm10}
    \end{subfigure}
    \begin{subfigure}[t]{0.32\columnwidth}
        \centering
        \caption{\panda}
        \label{fig:runtime-mvh-panda}
    \end{subfigure}

    \caption{Runtime relative to \namoadr as a function of MVH size for each algorithm. Solid and dashed lines correspond to (C)-type and (A)-type MVHs, respectively. \ny: 264,346 nodes, 733,846 edges, 3 objectives, \netm: 10,000 nodes, \small 59,743 edges, 3 objectives, \panda: 1,000 nodes, \small 10,260 edges, 8 objectives.}
    \label{fig:fig2-runtime-vs-mvh-size}
\end{figure}

\begin{figure}[htbp]
    \centering
    
    \begin{subfigure}[t]{\columnwidth}
        \centering
        \includegraphics[width=0.8\columnwidth]{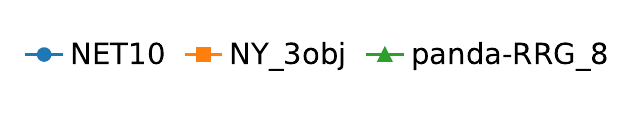}
    \end{subfigure}

    \vspace{0.2em}

    \begin{subfigure}[t]{\columnwidth}
        \centering
        \includegraphics[width=0.5\linewidth]{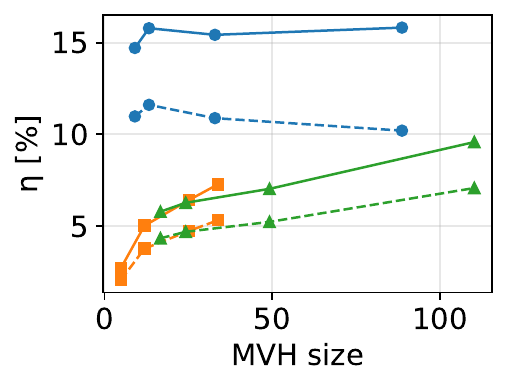}
    \end{subfigure}

    \caption{
    Frequency $\eta$, measured in percentage for which \frankenmvh performs full dominance checks out of all local (solid line) and global (dashed line) dominance checks, as a function of (A)-type MVHs.
    }

    \label{fig:ratio-vs-mvh-size}
\end{figure}
\begin{figure}[t]
    \centering

    \includegraphics[width=\columnwidth]{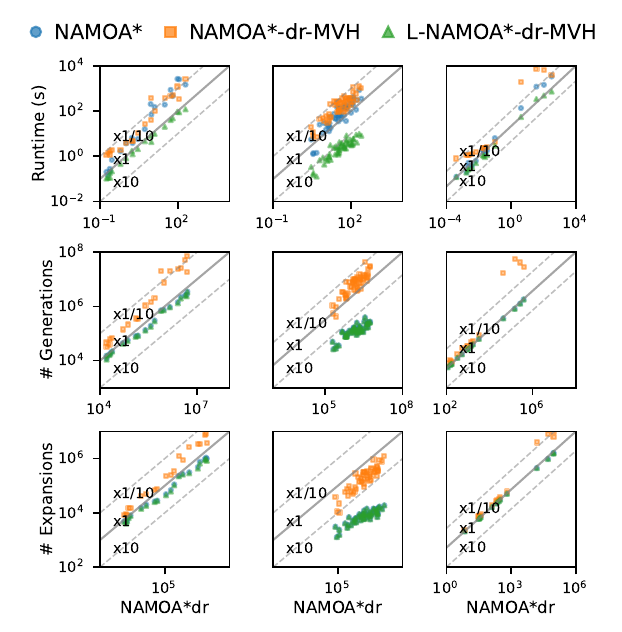}

    \vspace{-1em}
    \hspace{0.2em}
    \begin{subfigure}[t]{0.32\columnwidth}
        \centering
        \caption{\ny}
        \label{fig:node-effort-ny}
    \end{subfigure}
    \hfill
    \begin{subfigure}[t]{0.32\columnwidth}
        \centering
        \caption{\netm}
        \label{fig:node-effort-netm10}
    \end{subfigure}
    \hfill
    \begin{subfigure}[t]{0.32\columnwidth}
        \centering
        \caption{\panda}
        \label{fig:node-effort-panda}
    \end{subfigure}

\caption[Runtime and node-effort comparison]{
Per-query comparison against \namoadr\ across the three benchmarks.
Rows show runtime, node generations, and node expansions; columns correspond to benchmarks (see Fig.~\ref{fig:fig2-runtime-vs-mvh-size}).
Runtime results use the best-performing MVH for each algorithm.
For node generations and expansions, a single MVH is fixed per benchmark.
The solid diagonal marks parity with \namoadr, and dashed diagonals mark $10\times$ differences.
}
    \label{fig:node-effort-benchmarks}
\end{figure}

We continue in Fig.~\ref{fig:node-effort-benchmarks} to illustrate the possible speedup for each algorithm variant, together with their number of generated and expanded nodes. Every algorithm is selected with its best-performing MVH.
\frankenmvh\ provides a consistent speedup or similar runtime to \namoadr, often exceeding one order of magnitude ($>10\times$) in the \netm\ benchmark.
In contrast, \namoa\ and \soundmvh\ perform significantly slower, emphasizing the necessity of DR optimizations upon its lack or relaxed adaptation.

Also in Fig~\ref{fig:node-effort-benchmarks}, we can see that the operation counts closely align with the runtime results.
In particular, \soundmvh, which creates separate nodes for different heuristic values, causes a state-space blow-up.
By contrast, \frankenmvh's generated and expanded node counts are exactly the same as those of \namoa, so the corresponding marks collapse together.
This shows that \frankenmvh\ preserves the advantage of MVH guidance while avoiding the cost of eagerly materializing all heuristic bindings.

\paragraph{Effectiveness of Laziness.}
Figure~\ref{fig:ratio-vs-mvh-size} illustrates the frequency for which \frankenmvh requires performing full dominance check instead of using DR. The very low percentages across all cases ($16\%$ at most) highlight the efficiency of applying lazy DR dominance checks, when only occasionally falling back to doing full-dominance check.

\section{Conclusion \& Future Work}
\label{sec:conclusion}

In this paper, we addressed the fundamental incompatibility between MVHs and DR in MOS. By tracing how heuristic variance destroys the monotonic growth invariant required for safe $t$-discarding, we established the first theoretical frameworks to restore search correctness. While our baseline (\soundmvh) proved that consistency restores the invariant at the cost of a state-space explosion, our primary contribution, \frankenmvh, successfully avoids this bottleneck. By falling back to full dominance checks only when local sequence violations are detected, \frankenmvh\ preserves the efficiency of DR while exploiting the stronger guidance of MVHs, achieving order-of-magnitude speedups in instances where this guidance translates into stronger pruning.

This work opens several directions for MOS, including the design of domain-independent MVH construction methods that produce tighter, compact admissible sets for general graphs.
Additionally, the optimistic, lazy-validation mechanics introduced in \frankenmvh\ are highly modular; future work will explore integrating these techniques into other state-of-the-art multi-objective frameworks, to further extend the reach and efficiency of heuristic-guided MOS algorithms.
\clearpage

\bibliography{references}

\ifarxivappendix
\clearpage
\appendix
\section{Supplementary material for Sec.~4 (\soundmvh)}
\label{app:sec:soundmvh}

Recall that the variability of heuristic estimates within an MVH destroys the ordering invariant required for DR. However, in Sec.~4 we show that if we group nodes not merely according to their state $s$, but according to state-heuristic pairs $(s, \vecv{h})$, we can isolate the variance and restore the invariant.

To achieve this, the search algorithm must guarantee that a path maintains a consistent heuristic sequence from start to goal. We formalize this requirement as follows:

\begin{definition}[Path-Consistent Heuristic Selection]
\label{app:def:mvh-sound}
A MOS algorithm employs path-consistent heuristic selection if, whenever a node $n$ with state $s$ generates a successor $n'$ with state $s'$, the chosen heuristic vector $\vecv{h}(n') \in H(s')$ satisfies $\vecv{h}(n) \preceq \vecv{c}(s,s') + \vecv{h}(n')$.
\end{definition}

This allows us to recover the monotonicity required for~DR.
\begin{lemma}[$\vecv{f}$-value monotonicity along a path]
\label{app:lem:f-monotonicity}
Assume~$H$ is a consistent MVH and the algorithm employs path-consistent heuristic selection. Then along any generated path, the sequence of evaluation vectors $\vecv{f} = \vecv{g} + \vecv{h}$ is component-wise non-decreasing.
\end{lemma}

\begin{proof}
By Def.~\ref{app:def:mvh-sound}, every generated successor node $n'$ with state $s'$ satisfies $\vecv{h}(n) \preceq \vecv{c}(s,s') + \vecv{h}(n')$. Adding $\vecv{g}(n)$ to both sides, and using the fact that $\vecv{g}(n') = \vecv{g}(n) + \vecv{c}(s,s')$, yields $\vecv{g}(n) + \vecv{h}(n) \preceq \vecv{g}(n') + \vecv{h}(n')$, which is $\vecv{f}(n) \preceq \vecv{f}(n')$.
\end{proof}

\begin{lemma}
\label{app:lem:lex-f-extraction}
Assume $H$ is a consistent MVH and the algorithm employs path-consistent heuristic selection. Then, nodes are extracted from \open{} in lexicographically non-decreasing order of their $\vecv{f}$-values.
\end{lemma}

\begin{proof}
Each extraction removes the lexicographically smallest $\vecv{f}$ from \open{}. By Lemma~\ref{app:lem:f-monotonicity}, every newly generated successor node has an $\vecv{f}$-value that is lexicographically greater than or equal to that of its parent. Therefore, any later insertion into \open{} cannot violate the extraction order.
\end{proof}

By forcing the algorithm to be path-consistent, we guarantee that among nodes sharing the exact same state \emph{and} heuristic vector, the $\vecv{g}$-values are extracted in the correct order for $t$-discarding.

\begin{theorem}[Restored $\vecv{g}$-value ordering]
\label{app:thm:restored-g-order}
Assume $H$ is a consistent MVH and the algorithm employs path-consistent heuristic selection. Let $n$ and $n'$ be two extracted nodes that share the same state $s$ and the exact same selected heuristic vector $\vecv{h}$. If $n$ is extracted before $n'$, then $\vecv{g}(n) \le_{\rm lex} \vecv{g}(n')$.
\end{theorem}

\begin{proof}
By Lemma~\ref{app:lem:lex-f-extraction}, since $n$ is extracted before $n'$, we know $\vecv{f}(n) \le_{\rm lex} \vecv{f}(n')$. Because both nodes share the exact same heuristic vector $\vecv{h}$, subtracting $\vecv{h}$ component-wise from both evaluation vectors preserves the lexicographic order: $\vecv{g}(n) = \vecv{f}(n) - \vecv{h} \le_{\rm lex} \vecv{f}(n') - \vecv{h} = \vecv{g}(n')$.
\end{proof}

\section{Supplementary material for Sec.~5 (\frankenmvh : Lazy DR)}
\label{app:sec:frankenmvh}
The correctness of \frankenmvh\ relies on proving that these lazy evaluations and fallback mechanisms are completely sound.

\begin{lemma}[Soundness of \textsc{ChooseH}]
\label{app:lem:chooseh-sound}
Let a generated node be identified by $n=(s, \vecv{g}, \vecv{f})$. If \textsc{ChooseH}$(s,\vecv{g}, \GTcl{\sgoal})$ returns $\bot$, then for every heuristic vector $\vecv{h}\in H(s)$, the resulting evaluation vector $\vecv{f} = (\vecv{g}+\vecv{h})$ is $t$-discarded by $\GTcl{\sgoal}$. Hence, no completion of $n$ can yield a Pareto-optimal solution.
\end{lemma}

\begin{proof}
\textsc{ChooseH} scans all heuristic vectors $\vecv{h}\in H(s)$ and returns the lexicographically-first vector which is not $t$-discarded by $T$. When calling \textsc{ChooseH} against~$\GTcl{\sgoal}$, returning $\bot$ means that every admissible heuristic choice is discarded by an existing solution in~$\GTcl{\sgoal}$. Since $H$ is admissible, each $\vecv{f} = \vecv{g}+\vecv{h}$ is a lower bound on any completion through $s$. Therefore, if all such evaluations are dominated by already discovered solution costs, no continuation of $\vecv{g}$ can produce a new Pareto-optimal solution.
\end{proof}

\begin{lemma}[Soundness of \textsc{LocalDomCheck}]
\label{app:lem:franken-local-full}
The two stages of pruning executed within \textsc{LocalDomCheck} are sound.
\end{lemma}

\begin{proof}
Let a generated node be identified by $n=(s, \vecv{g}, \vecv{f})$. If $\GTcl{s}\not\prec\Tr(\vecv{g})$, then $G_{\mathrm{cl}}(s) \not\prec \vecv{g}$ must also be true, and the node is correctly not pruned (Line~42). If $\GTcl{s} \preceq \Tr(\vecv{g})$, we divide the proof into cases:
If $\max\{g'_1 \mid \vecv{g}' \in \GTcl{s}\} \le g_1$, then every previously extracted node with state $s$ has a first-component cost of at most $g_1$, which is the exact condition required for $t$-discarding correctness.
If $\max\{g'_1 \mid \vecv{g}' \in \GTcl{s}\} > g_1$ then $g$ is not dominated by discovered vectors in $\GTcl{s}$, and pruning is evaluated against the full discovered Pareto frontier discovered in the closed-set $G_{\mathrm{cl}}(s)$ (Line~45). In this final case, pruning is performed exactly as in the standard, proven \namoa\ dominance check.
\end{proof}

\begin{theorem}[Correctness of \frankenmvh{}]
\label{app:thm:franken-correct}
Assume $H$ is an admissible multi-valued heuristic. Then \frankenmvh\ returns the complete Pareto-optimal solution set.
\end{theorem}
\begin{proof}
A generated node $n=(s, \vecv{g}, \vecv{f})$ may be discarded in only two ways. First, it may fail the global check. By Lemma~\ref{app:lem:chooseh-sound}, this can happen only after all heuristic choices have been exhausted, in which case no completion of $\vecv{g}$ can lead to a Pareto-optimal solution. Second, it may be removed by the local check, which is proven sound by Lemma~\ref{app:lem:franken-local-full}. Thus, every discarded node is truly dominated, while every non-dominated path remains eligible for extraction. Since the search is best-first and $H$ is admissible, every Pareto-optimal solution is eventually generated and inserted into $\sols$. Therefore, the returned set is exactly the Pareto-optimal solution set $\Pi^*$.
\end{proof}

\fi

\end{document}